\documentclass[11pt, letterpaper]{article} %
\usepackage{times}
\usepackage{url}            %
\usepackage{booktabs}       %
\usepackage{amsfonts}       %
\usepackage{nicefrac}       %
\usepackage{microtype}      %
\usepackage{xspace}
\usepackage{multirow}
\usepackage{amsmath}
\usepackage{algorithm}
\usepackage{algorithmic}
\usepackage{color, colortbl}
\usepackage{enumitem}
\usepackage{comment}
\usepackage{bm}
\usepackage{fullpage}
\usepackage{caption}
\usepackage{subcaption}
\usepackage{amssymb}
\usepackage{mathtools}
\usepackage{amsthm}
\usepackage[dvipsnames]{xcolor}
\usepackage[colorlinks=true, linkcolor=red, citecolor=blue, urlcolor=WildStrawberry]{hyperref}
\usepackage[capitalize,noabbrev]{cleveref}
\usepackage{natbib}
\usepackage{mathrsfs}
\usepackage{parskip}

\usepackage{mkolar_definitions}

\input{rwstyle.tex}

\newcommand{\alglinelabel}{
  \addtocounter{ALC@line}{-1}
  \refstepcounter{ALC@line}
  \label
}

\Crefname{assumption}{Assumption}{Assumptions}

\author{%
Ayush Sekhari$^{1}$ \quad Karthik Sridharan$^{2}$ \quad Wen Sun$^{2}$ \quad Runzhe Wu$^{2}$
\vspace{5pt}
\\
\normalsize{$^1$MIT \quad $^2$Cornell University}
\vspace{5pt}
\\
\small{\texttt{sekhari@mit.edu \quad \{ks999,ws455,rw646\}@cornell.edu}
}}

\date{}

\title{Contextual Bandits and Imitation Learning via Preference-Based Active Queries\thanks{Authors are listed in alphabetical order of their last names.}} 

\bibpunct{(}{)}{;}{a}{,}{,}

\usepackage[suppress]{color-edits} 
\addauthor{as}{red}
\addauthor{rw}{orange}

\begin{document}

\sloppy 

\maketitle 

\begin{abstract}
We consider the problem of contextual bandits and imitation learning, where the learner lacks direct knowledge of the executed action's reward. Instead, the learner can actively query an expert at each round to compare two actions and receive noisy preference feedback. The learner's objective is two-fold: to minimize the regret associated with the executed actions, while simultaneously, minimizing the number of comparison queries made to the expert. In this paper, we assume that the learner has access to a function class that can represent the expert's preference model under appropriate link functions, and provide an algorithm that leverages an online regression oracle with respect to this function class \asedit{for choosing its actions and deciding when to query}. For the contextual bandit setting, our algorithm achieves a regret bound that combines the best of both worlds, scaling as \(O(\min\{\sqrt{T}, d/\Delta\})\), where \(T\) represents the number of interactions, \(d\) represents the eluder dimension of the function class, and \(\Delta\) represents the minimum preference of the optimal action over any suboptimal action under all contexts. Our algorithm does not require the knowledge of \(\Delta\), and the obtained regret bound is comparable to what can be achieved in the standard contextual bandits setting where the learner observes reward signals at each round. Additionally, our algorithm makes only \(O(\min\{T, d^2/\Delta^2\})\) queries to the expert. 

We then extend our algorithm to the imitation learning setting, where the learning agent engages with an unknown environment in episodes of length \(H\) each, and provide similar guarantees for regret and query complexity. \asedit{The regret bound for our imitation learning algorithm, which relies on preference-based feedback, matches the prior results in interactive imitation learning \citep{ross2014reinforcement} that require access to the expert's actions as well as reward signals. Furthermore, we show that our algorithm enjoys improved query complexity bounds. Interestingly, in some cases, our algorithm for imitation learning via preference-feedback can even learn to outperform the underlying expert thus highlighting a practical benefit of considering preference-based feedback in imitation learning.}  
 
\end{abstract} 

\section{Introduction}

Human feedback for training machine learning models has been widely used in many scenarios, including robotics \citep{ross2011reduction,ross2013learning,jain2015learning,laskey2016shiv,christiano2017deep} and natural language processing \citep{stiennon2020learning,ouyang2022training}. By integrating human feedback into the training process,  these prior works provide techniques to align machine-learning models with human intention and enable high-quality human-machine interaction (e.g., ChatGPT). 

Existing methods generally leverage two types of human feedback. The first is the action from human experts, which is the dominant feedback mode used in the literature of imitation learning or learning from demonstrations \citep{abbeel2004apprenticeship,ziebart2008maximum,daume2009search,ross2011reduction,ross2014reinforcement,sun2017deeply,osa2018algorithmic,li2023reinforcement}. 
The second type of feedback is preference-based feedback, which involves comparing pairs of actions. In this approach, the expert provides feedback by indicating their preference between two options selected by the learner. While both types of feedback have their applications, our focus in this work is on preference-based feedback, which is particularly suitable for scenarios where it is challenging for human experts to recommend the exact optimal action while making pairwise comparisons is much easier. 

Learning via preference-based feedback has been extensively studied, particularly in the field of \textit{dueling bandits} \citep{yue2011beat,yue2012k,zoghi2014relative,ailon2014reducing,komiyama2015regret,wu2016double,saha2021dueling,bengs2021preference,saha2022versatile} and \textit{contextual dueling bandits} \citep{dudik2015contextual,saha2022efficient,wu2023borda}.  Different from the standard bandit setting, the learner proposes two actions in dueling bandits and only gets noisy preference feedback from the human expert. Follow-up works extend the preference-based learning model from the one-step bandit setting to the multi-step decision-making (e.g., IL and RL) setting \citep{chu2005preference,sadigh2017active,christiano2017deep,lee2021pebble,chen2022human,saha2023dueling}. These studies mainly focus on how to learn a high-quality policy from human feedback, without concerning the question of active query in order to minimize the query complexity.  

However, query complexity is an important metric to optimize when learning from human feedback, as human feedback is expensive to collect \citep{lightman2023let}. For instance, InstructGPT \citep{ouyang2022training} is trained only on around 30K pieces of human feedback, which is significantly fewer than the internet-scale dataset used for pre-training the base model GPT3, indicating the challenge of scaling up the size of human feedback datasets. In other areas, such as robotics, learning from human feedback is also not easy, and prior studies (e.g., \citet{cohn2011comparing,zhang2022time,myers2023active}) have explored this issue from various perspectives. \citet{ross2013learning,laskey2016shiv} pointed out that querying human feedback in the learning loop is challenging, and extensively querying for feedback puts too much burden on the human experts. 

In this work, we design \emph{principled algorithms that can learn from preference-based feedback while at the same time minimizing query complexity}, 
under the settings of contextual bandits \citep{auer2002nonstochastic,langford2007epoch} and imitation learning \citep{ross2011reduction}. Our main contributions can be summarized as follows.

\begin{itemize}[leftmargin=*]
	\item In the contextual dueling bandits setting, the stochastic preference feedback is generated based on some preference matrix \citep{saha2022efficient}. We propose an algorithm (named \textsc{AURORA} -- in short of \textit{Active preference qUeRy fOR contextual bAndits}) that can achieve a best-of-both-worlds regret bound (i.e., achieves the minimum of the worst-case regret and an instance dependent regret), while at the same providing an instance-dependent query complexity bound. For benign instances with small eluder dimension and large gap,  our regret and query complexity bounds both scale with $\ln(T)$ where $T$ is the total number of interactions in contextual bandits. 
	\item In imitation learning, the stochastic preference feedback is generated based on the underlying reward-to-go of the expert's policy (e.g., the expert prefers actions that lead to higher reward-to-go). We propose an algorithm named \textsc{AURORAE}, in short of \textit{Active preference qUeRy fOR imitAtion lEarning}, which instantiates $H$ instances of \textsc{AURORA}, one per each time step for the finite horizon Markov Decision Process (MDP), where $H$ is the horizon. By leveraging preference-based feedback, we show that, interestingly, our algorithm can learn to outperform the expert when the expert is suboptimal. Such a result is beyond the scope of the classic imitation learning algorithm \textsc{DAgger}, and previously can only be achieved by algorithms like \textsc{AggreVaTe(d)} \citep{ross2014reinforcement,sun2017deeply,cheng2018convergence} and \textsc{LOLS} \citep{chang2015learning} which require direct access to expert's actions and also reward signal -- a much stronger feedback mode than ours.
\end{itemize}

To the best of our knowledge, for both contextual bandit and imitation learning with preference-based feedback, our algorithms are the first to achieve best-of-both-worlds regret bounds via active querying. %

\subsection{Related works}
\paragraph{Selective Sampling.}
Numerous studies have been conducted on selective sampling across various settings \citep{cesa2005minimizing,dekel2012selective,agarwal2013selective,hanneke2015minimax,hanneke2021toward,zhu2022efficient}, with the work of \citet{sekhari2023selective} being closest to ours. \citet{sekhari2023selective} presented a suite of provably efficient algorithms that are applicable to settings including contextual bandits and imitation learning. The primary distinction between our setting and the prior works lies in the feedback modality--we assume preference-based feedback, whereas they assume direct label feedback or reward signals. 

\paragraph{Contextual bandits with preference feedback.}
\citet{dudik2015contextual} is the first to consider contextual dueling bandits, and one of their algorithms achieves the optimal regret rate. \citet{saha2022efficient} studied contextual dueling bandits using a value function class and proposed an algorithm based on a reduction to online regression, which also achieves an optimal worst-case regret bound. 
In this paper, we mainly follow the setting of the latter and make notable improvements in two aspects: (1) in addition to the $O(\sqrt{AT})$ optimal regret rate where $A$ is the number of actions and $T$ is the number of interaction rounds, we established an instance-dependent regret upper bound that can be significantly smaller when the bandit exhibits a favorable structure; (2) our algorithm has an instance-dependent upper bound on the number of queries, and thus when the underlying instance is well behaved (has small eluder dimension and large gap), we will make significantly fewer queries.

Another related work is \citet{saha2022versatile} which achieves the best-of-both-worlds regret for non-contextual dueling bandits. We note that our setting is more general due to the existence of context and general function approximation, enabling us to leverage function class beyond linear and tabular cases.

\paragraph{RL with preference feedback.} RL with preference feedback has been widely employed in recent advancements in AI \citep{ouyang2022training,openai2023gpt}. According to \citet{wirth2017survey}, there are generally three types of preference feedback: action preferences \citep{furnkranz2012preference}, state preferences \citep{wirth2014learning}, and trajectory preferences \citep{busa2014preference,novoseller2020dueling,xu2020preference,lee2021b,chen2022human,saha2023dueling,pacchiano2021dueling,biyik2018batch,taranovic2022adversarial,sadigh2017active}. We focus on the action preference modality with the goal of achieving tight regret bounds and query complexities.
 
The concurrent work from \citet{zhan2023query} investigates the experimental design in both the trajectories-based and action-based preference settings, for which they decouple the process of collecting trajectories from querying for human feedback. Their action-based setting is the same as ours, but they mainly focus on linear parameterization, while our approach is a reduction to online regression and can leverage general function approximation beyond linear function classes.

\paragraph{Imitation learning.} In imitation learning, two common feedback modalities are typically considered: demonstrations that contain experts' actions, and preferences. The former involves directly acquiring expert actions (e.g., \citet{ross2011reduction,ross2014reinforcement,sun2017deeply, chang2015learning, sekhari2023selective}), while the latter focuses on obtaining preferences between selected options \citep{chu2005preference,lee2021pebble,zhu2023principled}.  \citet{brown2019extrapolating,brown2020better} leveraged both demonstrations and preference-based information and empirically showed that their algorithm can learn to outperform experts. Our imitation learning setting belongs to the second category, and we established bounds on the regret and the query complexity for our algorithm. We show that our algorithm can learn a policy that can provably outperform the expert (when it is suboptimal for the underlying environment).

\section{Preliminaries}

In this section, we introduce the setup for contextual bandits and imitation learning with preference-based feedback. 
We denote $[N]$ as the set of integers \(\crl{1, \dots, N}\). The set of all distributions over a set $\+S$ is denoted by $\Delta(\+S)$.

\subsection{Contextual Bandits with Preference-Based Feedback}
In this section, we introduce the contextual dueling bandits setting. We assume a context set $\+X$ and an action space $\+A=[A]$. At each round $t\in[T]$, a context $x_t$ is drawn \textit{adversarially}, and the learner's task is to decide whether to make a query \asedit{to the expert}. If the learner makes a query, it needs to select a pair of actions $(a_t,b_t)\in\+A\times\+A$, upon which a noisy feedback $y_t\in\{-1,1\}$ is revealed to the learner regarding whether $a_t$ or $b_t$ is better. Specifically, we assume \asedit{that the expert} relies on a preference function $f^\star:\+X\times\+A\times\+A\rightarrow[-1,1]$ \asedit{based on which, it samples its feedback $y_t$ as} 
\begin{align*}
\Pr(a_t\text{ is preferred to }b_t\given x_t)
\coloneqq 
\Pr( y_t = 1 \given x_t, a_t,b_t )
=
\phi\big(f^\star(x_t,a_t,b_t)\big)
\end{align*}
where $\phi(d):[-1,1]\rightarrow[0,1]$ is the link function, which satisfies $\phi(d)+\phi(-d)=1$ for any $d$. If the learner does not make a query, it should still select a pair of actions $(a_t,b_t)\in\+A\times\+A$ but will not receive any feedback.
Let $Z_t\in\{0,1\}$ indicate whether the learner makes a query at round $t$.

We assume that the learner has access to a function class $\+F\subseteq\+X\times\+A\times\+A\rightarrow[-1,1]$ that realizes \(f^\star\). Furthermore, we assume that $f^\star$, as well as the functions in $\+F$, is transitive and anti-symmetric.
\begin{assumption}\label{asm:properties}
	We assume $f^\star\in\+F$ and any functions $f\in\+F$ satisfies the following two properties:
		(1) transitivity: for any $x\in\+X$ and $a,b,c\in\+A$, if $f(x,a,b)>0$ and $f(x,b,c)>0$, then we must have $f(x,a,c)>0$;
		(2) anti-symmetry: $f(x,a,b)=-f(x,b,a)$ for any $x\in\+X$ and any $a,b\in\+A$.
\end{assumption}  
We provide an example below for which \Cref{asm:properties} is satisfied.
\begin{example}\label{lem:ex-r-r}
Assume there exists a function $r^\star:\+X\times\+A\rightarrow[0,1]$ such that $f^\star(x,a,b)=r^\star(x,a)-r^\star(x,b)$ for any $x\in\+X$ and $a,b\in\+A$. Typically, such a function $r^\star$ represents the ``reward function'' of the contextual bandit. In such a scenario, we can first parameterize a reward class $\+R\subseteq\+X\times\+A\rightarrow[0,1]$ and define $\+F=\{f:f(x,a,b)=r(x,a)-r(x,b), r\in\+R\}$. Moreover, it is common to have $\phi(d)\coloneqq 1/(1+\exp(-d))$ in this setting, which recovers the Bradley-Terry-Luce (BTL) model \citep{bradley1952rank}---a commonly used model in practice for learning reward models \citep{christiano2017deep}.
\end{example}
\Cref{asm:properties} ensures the existence of an optimal arm, as stated below.
\begin{lemma}\label{lem:optimal-action-exist}
	 Under \Cref{asm:properties}, for any function $f\in\+F$ and any context $x\in\+X$, there exists an arm $a\in\+A$ such that $f(x,a,b)\geq0$ for any arm $b\in\+A$. We denote this best arm by $\pi_f(x):=a$.\footnote{When the best arms is not unique, the ties are broken arbitrarily but consistently.}

\end{lemma}
The learner's goal is to minimize the regret while minimizing the number of queries, which are defined as:
\begin{align*}
		\@{Regret}_T^{\@{CB}}\coloneqq\sum_{t=1}^{T}\big(
		f^\star(x_t,\pi_{f^\star}(x_t),a_t)+
				f^\star(x_t,\pi_{f^\star}(x_t),b_t)\big)
            ,\quad
		\@{Queries}_T^{\@{CB}}\coloneqq\sum_{t=1}^T Z_t.
\end{align*}

It is worth noting that when $f^\star$ is the difference in rewards (as in \Cref{lem:ex-r-r}), the regret defined above reduces to the standard regret of a contextual bandit.
We also remark that our feedback generation generalizes that of \citet{saha2022efficient} in that we assume an additional link function $\phi$, while they assume the feedback is sampled from $\Pr(y=1\given x,a,b)=(P_t[a_t,b_t]+1)/2$, which is captured in our setting (see \Cref{ex:sq-loss}). However, \cite{saha2022efficient} do not assume transitivity.

\subsection{Imitation Learning with Preference-Based Feedback}

In our imitation learning setup, we consider that the learner operates in a finite-horizon Markov decision process (MDP), which is a tuple $M(\+X,\+A,r,P,H)$ where $\+X$ is the state space, $\+A$ is the action space, $P$ is the transition kernel, $r:\+X\times\+A\rightarrow[0,1]$ is the reward function, and $H$ is the length of each episode. The interaction between the learner and the environment proceeds as follows: at each episode $t\in[T]$, the learner receives an initial state $x_{t,0}$ which could be chosen adversarially. Then, the learner interacts with the environment for $H$ steps. At each step $h$, the learner first decides whether to make a query. If making a query, the learner needs to select a pair of actions $(a_{t,h},b_{t,h})\in\+A\times\+A$, upon which a feedback $y_{t,h}\in\{-1,1\}$ is revealed to the learner regarding which action is preferred from the expert's perspective. Here the feedback is sampled according to
\begin{align*}
\Pr(a_{t,h}\text{ is preferred to }b_{t,h}\given x_{t,h}, h)
\coloneqq
\Pr( y_{t,h} = 1 \given x_{t,h}, a_{t,h}, b_{t,h}, h )
=
\phi\big(f_h^\star(x,a_{t,h},b_{t,h})\big).
\end{align*}
Irrespective of whether the learner made a query, it then picks a single action from $a_{t,h},b_{t,h}$ and transit to the next step (our algorithm will just pick an action uniformly at random from $a_{t,h},b_{t,h}$). After $H$ steps, the next episode starts. Let $Z_{t,h}\in\{0,1\}$ indicate whether the learner decided to query at step \(h\) in episode $t$. We assume that the function class $\+F$ is a product of $H$ classes, i.e., $\+F=\+F_0\times\cdots\+F_{H-1}$ where, for each $h$, we use $\+F_h=\{f:\+X\times\+A\times\+A\rightarrow[-1,1]\}$ to model $f^\star_h$ and assume that $\+F_h$ satisfies \Cref{asm:properties}. 

A policy is a mapping $\pi:\+X\rightarrow\Delta(\+A)$. For a policy $\pi$, the state value function for a state $x$ at step $h$ is defined as $V^\pi_h(x)\coloneqq\E[\sum_{i=h}^{H-1} r_i\given x_h=x]$ and the state-action value function for a state-action pair $(x,a)$ is $Q^\pi_h(x,a)\coloneqq\E[\sum_{i=h}^{H-1} r_i\given x_h=x,a_h=a]$, where the expectations are taken w.r.t.~the trajectories sampled by $\pi$ in the underlying MDP. 

\asedit{In the imitation learning setting, we assume that the expert (who gives the preference-based feedback) is equipped with a markovian policy \(\pi_e\), and that the preference of the expert is dependent on the reward-to-go under \(\pi_e\) (i.e. on a state \(x\), actions with higher values of \(Q^{\pi_e}(s, a)\) will be preferred by the expert). Formalizing this intuition, we assume that $f^\star_h$ is defined such that as $f_h^\star(x,a,b)\coloneqq Q_h^{\pi_e}(x,a)-Q_h^{\pi_e}(x,b)$.  The goal of the learner is still to minimize the regret and number of queries:} 

\begin{align*}
		\@{Regret}_T^{\@{IL}}
		\coloneqq\sum_{t=1}^T \big(V^{\pi_e}_0(x_{t,0})-V^{\pi_t}_0(x_{t,0})\big),\quad
		\@{Queries}_T^{\@{IL}}\coloneqq\sum_{t=1}^T\sum_{h=0}^{H-1} Z_{t,h}.
\end{align*}
Here $\pi_t$ is the strategy the learner uses to select actions at episode $t$.

\subsection{Link Function and Online Regression Oracle}
Following the standard practice in the literature \citep{agarwal2013selective}, we assume $\phi$ is the derivative of some $\alpha$-strongly convex function (see \Cref{def:strong-conv}) $\Phi:[-1,1]\rightarrow\=R$ and define the associated loss function as $\ell_\phi(d,y)=\Phi(d)-d(y+1)/2$. Additionally, in line with prior works in the literature \citep{foster2020instance,foster2020beyond, simchi2022bypassing,foster2018practical,sekhari2023selective}, our algorithm utilizes an online regression oracle, which is assumed to have a sublinear regret guarantee w.r.t.~\(\+F\) on arbitrary data sequences. 
\begin{assumption}\label{asm:oracle-regret}
We assume the learner has access to an online regression oracle pertaining to the loss $\ell_\phi$ such that for any sequence $\{(x_1,a_1,b_1,y_1),\dots,(x_T,a_T,b_T,y_T)\}$ where the label $y_t$ is generated by $y_t\sim\phi(f^\star(x_t,a_t,b_t))$, we have
$$
	\sum_{t=1}^T\ell_\phi\big(f_t(x_t,a_t,b_t),y_t\big)-\inf_{f\in\+F}\ell_\phi\big(f(x_t,a_t,b_t),y_t\big)\leq\Upsilon(\+F,T)
$$
for some $\Upsilon(\+F,T)$ that grows sublinearly with respect to $T$.\footnote{The online regression oracle updates as follows: in each iteration, after seeing $x_t,a_t,b_t$, it proposes a decision $f_t$, then $y_t$ is revealed and the online regression oracle incurs loss $\ell_\phi(f_t(x_t,a_t,b_t), y_t)$.} For notational simplicity, whenever clear from the context, we define  $\Upsilon\coloneqq\Upsilon(\+F,T)$.  
\end{assumption}
Here $\Upsilon$ represents the regret upper bound and is typically of logarithmic order in $T$ or the cardinality of the function class $\+F$ in many cases (here we drop the dependence on $T$ in notation for simplicity). We provide a few examples below: \looseness=-1
\begin{example}[Squared loss]\label{ex:sq-loss}
If we consider $\Phi(d)=d^2/4+d/2+1/4$, which is $1/4$-strongly convex, then we obtain $\phi(d)=(d+1)/2$ and $\ell_\phi(d,y)=(d-y)^2/4$, thereby recovering the squared loss, which has been widely studied in prior works. For example, \citet{rakhlin2014online} characterized the minimax rates for online square loss regression in terms of the offset sequential Rademacher complexity, resulting in favorable bounds for the regret. Specifically, we have $\Upsilon= O(\log|\+F|)$ assuming the function class $\+F$ is finite, and $\Upsilon= O(d\log(T))$ assuming $\+F$ is a $d$-dimensional linear class. We also kindly refer the readers to \citet{krishnamurthy2017active,foster2018practical} for efficient implementations.
\end{example}
\begin{example}[Logistic loss]\label{ex:log-loss}
When $\Phi(d)=\log(1+\exp(d))$ which is strongly convex at $[-1,1]$, we have $\phi(d)=1/(1+\exp(-d))$ and $\ell_\phi(d,y)=\log(1+\exp(-yd))$. Thus, we recover the logistic regression loss, which allows us to use online logistic regression and achieve $\Upsilon=O(\log|\+F|)$ assuming finite $\+F$. There have been numerous endeavors in minimizing the log loss, such as \citet{foster2018logistic} and \citet[Chapter 9]{cesa2006prediction}.
\end{example}

\section{Contextual Bandits with Preference-Based Active Queries}

\begin{algorithm}[t]
\begin{algorithmic}[1]
   \caption{Active preference qUeRy fOR contextual bAndits (AURORA)}
   \label{alg:cb}
	\REQUIRE Function class $\+F$, confidence parameter $\beta=\frac{4\Upsilon}{\alpha}+\frac{16+24\alpha}{\alpha^2}\log\big(4\delta^{-1}\log(T)\big)$.
	\STATE Online regression oracle produces $f_1$.
   	\FOR{$t=1,2,\dots,T$}
   		\STATE Learner receives context $x_t$, and computes the version space 
   		\begin{align*}
   			\+F_t\gets\Bigg\{f\in\+F:\sum_{s=1}^{t-1} Z_s\Big(f(x_s,a_s,b_s)-f_s(x_s,a_s,b_s)\Big)^2\leq \beta\Bigg\}.
   		\end{align*}
   		and the candidate arm set $\+A_t\gets\{\pi_f(x_t):\forall f\in\+F_t\}$.
   		\STATE Learner decides whether to query $Z_t\gets\indic\{|\+A_t|>1\}$.
   		\IF{$Z_t=1$}  
   			\STATE Set $w_t\gets\sup_{a,b\in\+A_t}\sup_{f,f'\in\+F_t} f(x_t,a,b)-f'
   			(x_t,a,b)$, and $\lambda_t\gets\indic\{\sum_{s=1}^{t-1}Z_s w_s\geq\sqrt{AT/\beta}\}$.
   			\IF{$\lambda_t=0$}
   				\STATE $p_t\gets\@{Uniform}(\+A_t)$.
   			\ELSE
   				\STATE $\gamma_t\gets\sqrt{AT/\beta}$.
   				\STATE Let $p_t$ be a solution of $\max_{a\in\+A_t}\sum_b f_t(x_t,a,b)p_t(b)+\frac{2}{\gamma_t p_t(a)}\leq\frac{5A}{\gamma_t}$.\alglinelabel{line:compute-p}
   			\ENDIF
   			\STATE Learner samples $a_t,b_t\sim p_t$ independently and receives the feedback $y_t$. 
			\STATE Learner feeds $((x_{t},a_{t},b_{t}),y_{t})$ to the online regression oracle which returns $f_{t+1}$.
   		\ELSE
   			\STATE Learner sets $a_t$ and $b_t$ to be the only action in $\+A_t$, and plays them. \alglinelabel{line:no-query-play-only-arm}
                \STATE $f_{t+1}\gets f_t$.
   		\ENDIF	
   	\ENDFOR 
\end{algorithmic} 
\end{algorithm}

We first present our algorithm, named \textsc{AURORA}, for contextual dueling bandits, as shown in \Cref{alg:cb}. At each round $t\in[T]$, the online regression oracle outputs a predictor $f_t$, using which the learner constructs a version space $\+F_t$ containing all functions close to past predictors on observed data. Here, the threshold 
$	
\beta$ set to $4\Upsilon/\alpha+(16+24\alpha)\log\big(4\delta^{-1}\log(T)\big)/\alpha^2
$ ensures that $f^\star\in\+F_t$ for any $t\in[T]$ with probability at least $1-\delta$ (\Cref{lem:pointwise-bound}). Thus, $\+A_t$ is non-empty for all \(t \in [T]\) and correspondingly  Line~\ref{line:no-query-play-only-arm} is well defined. The learner then forms a candidate arm set $\+A_t$ consisting of greedy arms induced by all functions in the version space. When $|\+A_t|=1$, the only arm in the set is the optimal arm since $f^\star\in\+F_t$, and thus no query is needed ($Z_t=0$). However, when $|\+A_t|>1$, any arm in $\+A_t$ could potentially be the optimal arm, and thus the learner needs to make a comparison query to obtain more information. 

Next, we explain the strategy used by the learner for making a query. Firstly, the learner computes $w_t$, which represents the ``width'' of the version space. Specifically, $w_t$ overestimates the instantaneous regret for playing any arm in $\+A_t$ (\Cref{lem:regret-bounded-by-w}). Then, the learner defines $\lambda_t$ that indicates if the estimated cumulative regret  $\sum_{s=1}^{t-1}Z_w w_s$ has exceeded $\sqrt{AT/\beta}$. Note that $Z_t$ is multiplied to $w_t$ since no regret is incurred when $Z_t=0$. The strategy to choose the actions (to be queried) for different values of $\lambda_t$ are as follows: 

\begin{itemize}[leftmargin=*]
	\item If $\lambda_t=0$, the cumulative reward has not yet exceeded $\sqrt{AT/\beta}=O(\sqrt{T})$, so the learner will explore as much as possible by uniform sampling from $\+A_t$.
	\item If $\lambda_t=1$, the regret may have reached $O(\sqrt{T})$, and therefore the learner uses a technique similar to inverse gap weighting (IGW), as inspired by \cite{saha2022efficient}, to achieve a better balance between exploration and exploitation. Specifically, the learner solves the  convex program\footnote{It is convex as it can be written as $|\+A_t|$ convex constraints: $\sum_b f_t(x_t,a,b)p_t(b)+\frac{2}{\gamma_t p_t(a)}\leq\frac{5A}{\gamma_t},\forall a\in\+A_t$.} in Line~\ref{line:compute-p}, which is feasible and whose solution $p_t$ satisfies (see \Cref{lem:igw})
	\begin{align*}
		\E_{a\sim p_t}\Big[f^\star(x_t,\pi_{f^\star}(x),a)\Big]
		=
		O\left(
		\gamma_t\E_{a,b\sim p_t}\Big[\big(f_t(x_t,a,b)-f^\star(x_t,a,b)\big)^2\Big]+\frac{A}{\gamma_t}\numberthis\label{eq:informal-igw}
		\right).
	\end{align*}
As a result of the above relation, we note that one can convert the instantaneous regret to the point-wise error between the predictor $f_t$ and the truth $f^\star$ plus an additive $A/\gamma_t$. This allows us to bound the cumulative point-wise error by the regret of the online regression oracle. In the special case, when there exists a ``reward function'' $r:\+X\times\+A\rightarrow[0,1]$ for each $f\in\+F$ such that $f(x,a,b)=r(x,a)-r(x,b)$ (\Cref{lem:ex-r-r}), the solution $p_t$ can be directly written as 
	\begin{align*}
		p_t(a)=
		\begin{cases} 
			\frac{1}{A+\gamma_t\big(r_t(x_t,\pi_{f_t}(x_t))-r_t(x_t,a)\big)} & a\neq\pi_{f_t}(x_t)\\
			1-\sum_{a'\neq\pi_{f_t}(x_t)}p_t(a') & a=\pi_{f_t}(x_t)
		\end{cases},
	\end{align*}
	where $r_t$ is the reward function associated with $f_t$, i.e., $f_t(x,a,b)=r_t(x,a)-r_t(x,b)$.
	This is the standard IGW exploration strategy \citep{foster2020beyond} and leads to the same guarantee as \eqref{eq:informal-igw} (see \Cref{lem:igw-r-version}).
\end{itemize}
\subsection{Theoretical Analysis}

Towards the theoretical guarantees of \Cref{alg:cb}, we employ two quantities to characterize a contextual bandit instance: the uniform gap and the eluder dimension, which are introduced below.
\begin{assumption}[Uniform gap]\label{asm:uniform-gap}
We assume the optimal arm $\pi_{f^\star}(x)$ induced by $f^\star$ under any context $x\in\+X$ is unique. Further, we assume a uniform gap $\Delta:=\inf_x \inf_{a\neq\pi_{f^\star}(x)} f^\star(x,\pi_{f^\star}(x),a)>0$.
\end{assumption}

We note that the existence of a uniform gap is a standard assumption in the literature of contextual bandits \citep{dani2008stochastic,abbasi2011improved,audibert2010best,garivier2019explore,foster2020beyond,foster2020instance}. Next, we introduce the eluder dimension \citep{russo2013eluder} and begin by defining ``$\epsilon$-dependence''.
\begin{definition}[$\epsilon$-dependence]
Let $\+G\subseteq\+X\rightarrow\=R$ be any function class. We say an element $x\in\+X$ is $\epsilon$-dependent on $\{x_1,x_2,\dots,x_n\}\subseteq\+X$ with respect to $\+G$ if any pair of functions $g,g'\in\+G$ satisfying $\sum_{i=1}^n(g(x_i)-g'(x_i))\leq\epsilon^2$ also satisfies $g(x)-g'(x)\leq\epsilon$. Otherwise, we say $x$ is $\epsilon$-independent of $\{x_1,x_2,\dots,x_n\}$.
\end{definition}

\begin{definition}[Eluder dimension]
	The $\epsilon$-eluder dimension of a function class $\+G\subseteq\+X\rightarrow\=R$, denoted by $\@{dim}_E(\+G,\epsilon)$, is the length $d$ of the longest sequence of elements in $\+X$ satisfying that there exists some $\epsilon'\geq\epsilon$ such that every element in the sequence is $\epsilon'$-independent of its predecessors.
\end{definition}
Eluder dimension is a standard complexity measure for function classes and has been used in the literature of bandits and RL extensively \citep{chen2022human,osband2014model,wang2020reinforcement,foster2020instance,wen2013efficient,jain2015learning,ayoub2020model,ishfaq2021randomized,huang2021towards}. Examples where the eluder dimension is small include linear functions, generalized linear models, and functions in Reproducing Kernel Hilbert Space (RKHS).  

Given these quantities, we are ready to state our main results. The proofs are provided in \Cref{sec:missing-pf}.

\begin{theorem}\label{thm:cb-regret}
Under \Cref{asm:properties,asm:oracle-regret,asm:uniform-gap}, \Cref{alg:cb} guarantees the following upper bounds of the regret and the number of queries:
\begin{align*}
	&\@{Regret}_T^{\@{CB}}=
	\widetilde{O}\left(
	\min\left\{
	\sqrt{AT\beta}
	,\;
	\frac{A^2\beta^2\@{dim}_E\left(\+F,\Delta\right)}{\Delta}
	\right\}
	\right),\\
	&\@{Queries}_T^{\@{CB}}=
	\widetilde{O}\left(
	\min\left\{T
	,\;
	\frac{A^3\beta^3 \@{dim}^2_E\left(\+F,\Delta\right)}{\Delta^2}
	\right\}
	\right)
\end{align*}
with probability at least $1-\delta$. We recall that $\beta=O(\alpha^{-1}\Upsilon+\alpha^{-2}\log(\delta^{-1}\log(T)))$, and $\alpha$ denotes the coefficient of strong convexity of $\Phi$. We have hidden logarithmic terms in the upper bounds for brevity. 
\end{theorem}
When the loss $\ell_\phi$ is either square loss or logistic loss (\Cref{ex:sq-loss,ex:log-loss}), the parameter $\beta$ is logarithmic in $T$. In such cases, the regret is $\~O(\min\{\sqrt{T},\@{dim}_E\left(\+F,\Delta\right)/\Delta\})$ and the number of queries is $\~O(\min\{T,\@{dim}^2_E(\+F,\Delta)/\Delta^2\})$, ignoring $A$ and logarithmic terms. Both consist of two components: the worst-case and the instance-dependent upper bounds. The worst-case bound provides a guarantee under all circumstances, while the instance-dependent one may significantly improve the upper bound when the underlying problem is well-behaved (i.e., has a small eluder dimension and a large gap).

\paragraph{Intuition of proofs.} We next provide intuition for why our algorithm has the aforementioned theoretical guarantees. First, we observe that from the definition of $\lambda_t$, the left term inside the indicator is non-decreasing, which allows us to divide rounds into two phases. In the first phase, $\lambda_t$ is always 0, and then at some point, it changes to 1 and remains 1 for the rest rounds. After realizing this, we first explain the intuition of the worst-case regret. In the first phase, as $w_t$ is an overestimate of the instantaneous regret (see \Cref{lem:regret-bounded-by-w}), the accumulated regret in this phase cannot exceed $O(\sqrt{T})$. In the second phase, we adapt the analysis of IGW to this scenario to obtain an $O(\sqrt{T})$ upper bound. A similar technique has been used in ~\cite{saha2022efficient,foster2020instance}. As the regret in both phases is at most $O(\sqrt{T})$, the total regret cannot exceed $O(\sqrt{T})$. Next, we explain the intuition of instance-dependent regret. Due to the existence of a uniform gap $\Delta$, we can first prove that as long as $|\+A_t|>1$, we must have $w_t\geq\Delta$ (see \Cref{lem:width-lower-bound}). This means that for all rounds that may incur regret, the corresponding width is at least $\Delta$. However, this cannot happen too many times as this frequency is bounded by the eluder dimension, which leads to an instance-dependent regret upper bound. Leveraging a similar technique, we can also obtain an upper bound on the number of queries. \looseness=-1

\paragraph{Comparion to \textsc{MinMaxDB} \citep{saha2022efficient}.}
In this prior work, the authors assume that $\Pr(y=1\given x,a,b)=(f^\star(x,a,b)+1)/2$, which is a specification of our feedback model (\Cref{ex:sq-loss}). 
While our worst-case regret bound matches their regret bound, our paper improves upon their results by having an additional instance-dependent regret bound that depends on the eluder dimension and gap. Furthermore, we also provide bounds on the query complexity which could be small for benign instances while \textsc{MinMaxDB} simply queries on every round. 

\paragraph{Comparion to \textsc{AdaCB} \citep{foster2020instance}.}
Our method shares some similarities with \citet{foster2020instance}, especially in terms of theoretical results, but differs in two aspects: (1) they assume regular contextual bandits where the learner observes the reward directly, while we assume preference feedback, and (2) they assume a stochastic setting where contexts are drawn i.i.d., but we assume that the context is adversarially chosen. While these two settings may not be directly comparable, it should be noted that \citep{foster2020instance} do not aim to minimize query complexity. 

\paragraph{Lower bounds.}
To understand whether our algorithm attains tight upper bounds, we provide the following lower bound which follows from a reduction from regular multi-armed bandits to contextual dueling bandits.  
\begin{theorem}[Lower bounds]\label{thm:lower-bound}
The following two claims hold:
\begin{enumerate}
\item[(1)] For any algorithm, there exists an instance that leads to $\@{Regret}^{\@{CB}}_T=\Omega(\sqrt{AT})$;
\item[(2)] For any algorithm achieving a worse-case expected regret upper bound in the form of $\E[\@{Regret}^{\@{CB}}_T]= O(\sqrt{AT})$, there exists an instance with gap $\Delta=\sqrt{A/T}$ that results in $\E[\@{Regret}^{\@{CB}}_T]=\Omega(A/\Delta)$ and $\E[\@{Queries}^{\@{CB}}_T]=\Omega(A/\Delta^2)=\Omega(T)$.
\end{enumerate}
\end{theorem}

By relating these lower bounds to \Cref{thm:cb-regret}, we conclude that our algorithm achieves a tight dependence on the gap $\Delta$ and $T$, up to logarithmic factors, in both the regret and query complexity upper bounds. Furthermore, as an additional contribution, we establish an alternative lower bound in Section \ref{sec:lb-2} by conditioning on the limit of regret, rather than the worst-case regret as assumed in Theorem \ref{thm:lower-bound}. %

\paragraph{Results without the uniform gap assumption.}
We highlight that \Cref{thm:cb-regret} can naturally extend to scenarios where a uniform gap does not exist (i.e., when \Cref{asm:uniform-gap} is not satisfied) without any modifications to the algorithm. The result is stated below, which is analogous to \Cref{thm:cb-regret}.
\begin{theorem}\label{thm:cb-general-regret}
Under \Cref{asm:properties,asm:oracle-regret}, \Cref{alg:cb} guarantees the following upper bounds of the regret and the number of queries:
\begin{align*}
	&\@{Regret}_T^{\@{CB}}=
	\widetilde{O}\left(
	\min\left\{
	\sqrt{AT\beta}
	,\,
	\min_{\epsilon>0}\left\{
	T_\epsilon\beta+\frac{A^2\beta^2\@{dim}_E\left(\+F,\epsilon\right)}{\epsilon}
	\right\}\right\}
	\right),\\
	&\@{Queries}_T^{\@{CB}}=
	\widetilde{O}\left(
	\min\left\{
	T
	,\,
	\min_{\epsilon>0}\left\{
	T_\epsilon^2\beta/A+\frac{A^3\beta^3 \@{dim}^2_E\left(\+F,\epsilon\right)}{\epsilon^2}
	\right\}	\right\}
	\right)
\end{align*}
with probability at least $1-\delta$. Here we define the gap of context $x$ as $\@{Gap}(x)\coloneqq\min_{a\neq\pi_{f^\star}(x)} f^\star(x,\pi_{f^\star}(x), a)$ and the number of rounds where contexts have small gap as $T_\epsilon\coloneqq\sum_{t=1}^T \indic\{\@{Gap}(x_t)\leq\epsilon\}$. We also recall that $\beta=O(\alpha^{-1}\Upsilon+\alpha^{-2}\log(\delta^{-1}\log(T)))$, and $\alpha$ denotes the coefficient of strong convexity of $\Phi$. %
\end{theorem}
Compared to \Cref{thm:cb-regret}, the above result has an extra gap-dependent term defined as $T_\epsilon$. Here $\epsilon$ denotes a gap threshold, and $T_\epsilon$ measures how many times the context falls into a small-gap region. We highlight that $T_\epsilon$ is small under certain conditions such as the Tsybakov noise condition \citep{tsybakov2004optimal}. It is also worth mentioning that our algorithm is agnostic to $\epsilon$, thus allowing us to take the minimum over all $\epsilon>0$. 

\paragraph{Comparion to \textsc{SAGE-Bandit} \citep{sekhari2023selective}.}
\Cref{thm:cb-general-regret} bears similarity to Theorem 4 in \citet{sekhari2023selective}, which examines active queries in contextual bandits with standard reward signal (\(0-1\) reward). It is worth noting that although our result looks slightly worse in terms of the factor $A$ (the number of actions), we believe that this inferiority is reasonable since our approach requires two actions to form a query, thus analytically expanding the action space to $\+A^2$. Whether this dependency can be improved remains a question for future investigation.

\section{Imitation Learning with Preference-Based Active Queries}

\begin{algorithm}[t]
\begin{algorithmic}[1]
\caption{Active preference qUeRy fOR imitAtion lEarning (AURORAE)}\label{alg:il}
\REQUIRE Function class  $\+F_0,\+F_1,\dots,\+F_{H-1}$, confidence parameter $\beta$.
\STATE Learner creates $H$ instances of \Cref{alg:cb}: $\textsc{AURORA}_h(\+F_h,\beta)$ for $h=0,1,\dots,H-1$.
   	\FOR{$t=1,2,\dots,T$}
   		\STATE Learner receive initial state $x_{t,0}$.
		\FOR{$h=0,1,\dots,H-1$}
			\STATE Learner feeds $x_{t,h}$ to $\textsc{AURORA}_h(\+F_h,\beta)$, and receives back 
		 $a_{t,h}$, $b_{t,h}$, $Z_{t,h}$.
			\IF{$Z_{t,h}=1$}
				\STATE Learner receives feedback $y_{t,h}$.
				\STATE Learner feeds $((x_{t,h},a_{t,h},b_{t,h}), y_{t,h})$
            to $\textsc{AURORA}_h(\+F_h,\beta)$ \asedit{to update its online regression oracle and local variables.} 
			\ENDIF
			\STATE Learner executes $a\sim \text{Uniform}(\{a_{t,h},b_{t,h}\})$ and transits to $x_{t,h+1}$. 
		\ENDFOR 
   	\ENDFOR 
\end{algorithmic}
\end{algorithm}

 In this section, we introduce our second algorithm, which is presented in \Cref{alg:il} for imitation learning. In essence, the learner treats the MDP as a concatenation of $H$ contextual bandits and runs an instance of \textsc{AURORA} (\Cref{alg:cb}) for each time step. Specifically, the learner first creates $H$ instances of \textsc{AURORA}, denoted by $\textsc{AURORA}_h$ (for $h=0,\dots,H-1$). Here, \textsc{AURORA}\(_h\) should be thought of as an interactive program that takes the context $x$ as input and outputs $a$, $b$, and $Z$. At each episode $t$, and each step $h$ therein, the learner first feeds the current state $x_{t,h}$ to $\textsc{AURORA}_h$ as the context; then, $\textsc{AURORA}_h$ decides whether to query (i.e.~\(Z_{t, h}\)) and returns the actions \(a_{t,h}\) and \(b_{t,h}\). If it decides to make a query, the learner will ask for the feedback $y_{t,h}$ on the proposed actions $a_{t,h}, b_{t,h}$, and provide the information $((x_{t,h},a_{t,h},b_{t,h}), y_{t,h})$ back to $\textsc{AURORA}_h$ to update its online regression oracle (and other local variables). We recall that the noisy binary feedback $y_{t,h}$ is sampled as  $y_{t,h}\sim \phi(Q^{\pi_e}_h(x_{t,h},a_{t,h})- Q^{\pi_e}_h(x_{t,h},b_{t,h}))$, and also emphasize that the learner neither has access to  $a\sim \pi_e(x_{t,h})$ like in \textsc{DAgger} \citep{ross2011reduction}  nor reward-to-go like in  \textsc{AggreVaTe(D)} \citep{ross2014reinforcement,sun2017deeply}. Finally, the learner chooses one of the two actions uniformly at random, executes it in the underlying MDP, and transits to the next state $x_{t,h+1}$ in the episode. The above process is then repeated with $\textsc{AURORA}_{h+1}$ till the episode ends. We name this algorithm \textsc{AURORAE}, the plural form of \textsc{AURORA}, which signifies that the algorithm is essentially a stack of multiple \textsc{AURORA} instances.

\subsection{Theoretical Analysis}

As \Cref{alg:il} is essentially a stack of \Cref{alg:cb}, we can inherit many of the theoretical guarantees from the previous section. To state the results, we first extend \Cref{asm:uniform-gap} into imitation learning. 
\begin{assumption}[Uniform Gap]\label{asm:uniform-gap-il} \asedit{Let \(f_h^\star\) be defined such that for any \(x \in \+X\), \(a, b \in \+A^2\), \(f_h^\star(x, a, b) = Q_h^{\pi_e}(x, a) -  Q_h^{\pi_e}(x, b)\).}
For all $h$, we assume the optimal action for $f_h^\star$ under any state $x\in\+X$ is unique. Further, we assume a uniform gap $\Delta:=\inf_h \inf_x \inf_{a\neq\pi_{f_h^\star}(x)} f_h^\star(x,\pi_{f_h^\star}(x),a)>0$.
\end{assumption}

This assumption essentially says that $Q^{\pi_e}_h$ has a gap in actions. 
We remark that, just as \Cref{asm:uniform-gap} is a common condition in the bandit literature, \Cref{asm:uniform-gap-il} is also common in MDPs \citep{du2019provably,foster2020instance,simchowitz2019non,jin2020simultaneously,lykouris2021corruption, he2021logarithmic}. The theoretical guarantee for \Cref{alg:il} is presented in \Cref{thm:il-regret}. We note a technical difference between this result and \Cref{thm:cb-regret}: although we treat the MDP as a concatenation of $H$ contextual bandits, the instantaneous regret of imitation learning is defined as the performance gap between the combined policy $\pi_t$ derived from the $H$ instances as a cohesive unit and the expert policy. This necessitates the use of performance difference lemma (\Cref{lem:pdl}) to get a unified result.

\begin{theorem}\label{thm:il-regret}
Under \Cref{asm:properties,asm:uniform-gap-il,asm:oracle-regret}, \Cref{alg:il} guarantees the following upper bounds of the regret and the number of queries:
\begin{align*}
	&\@{Regret}_T^{\@{IL}}\leq
	\widetilde{O}\left(
	H\cdot\min\left\{
	\sqrt{AT\beta}
	,\;
	\frac{A^2\beta^2\@{dim}_E\left(\+F,\Delta\right)}{\Delta}
	\right\}
	\right)-\@{Adv}_T,\\
	&\@{Queries}_T^{\@{IL}}\leq
	\widetilde{O}\left(
	H\cdot\min\left\{T
	,\;
	\frac{A^3\beta^3 \@{dim}^2_E\left(\+F,\Delta\right)}{\Delta^2}
	\right\}
	\right)
\end{align*}
with probability at least $1-\delta$. Here
$
\@{Adv}_T\coloneqq\sum_{t=1}^T \sum_{h=0}^{H-1} \E_{x_{t,h}\sim d^{\pi_t}_{x_{t,0},h}}[\max_a A^{\pi_e}_h(x_{t,h},a)]
$ is non-negative,
and $d^{\pi_t}_{x_{t,0},h}(x)$ denotes the probability of $\pi_t$ \footnote{Policy $\pi_t$ consists of $H$ time-dependent policies $\pi_{t,1},\dots, \pi_{t,H}$, where each $\pi_{t,h}$ is defined implicitly via $\textsc{AURORA}_h$, i.e., $\pi_{t,h}$ generates action as follows: given $x_{t,h}$, $\textsc{AURORA}_h$ recommends $a_{t,h},b_{t,h}$,  followed by uniformly sampling an action from $\{a_{t,h},b_{t,h}\}$. } reaching the state $x$ at time step $h$ starting from inital state $x_{t,0}$. In the above,  $\beta=O(\alpha^{-1}\Upsilon+\alpha^{-2}\log(H\delta^{-1}\log(T)))$ and $\alpha$ denotes the coefficient of strong convexity of $\Phi$.
\end{theorem}

Compared to \Cref{thm:cb-regret}, the main terms of the upper bounds for imitation learning are precisely the bounds in \Cref{thm:cb-regret} multiplied by $H$. In the proof presented in \Cref{sec:pf-thm-il-regret}, we use the performance difference lemma to reduce the regret of imitation learning to the sum of the regret of $H$ contextual dueling bandits, which explains this additional factor of $H$. 

Another interesting point is that the main term of the regret upper bound is subtracted by a non-negative term $\@{Adv}_T$, which measures the degree to which we can \textit{outperform} the expert policy. This means that our algorithm not only competes with the expert policy but can also surpass it to some extent. This guarantee is stronger than that of \textsc{DAgger} \citep{ross2011reduction} in that \textsc{DAgger} cannot ensure the learned policy is better than the expert policy regardless of how suboptimal the expert may be. While this may look surprising at first glance since we are operating under a somewhat weaker query mode than that of \textsc{DAgger}, we note that by querying experts for comparisons on pairs of actions with feedback sampling as $y\sim \phi( Q^{\pi_e}(x,a) - Q^{\pi_e}(x,b))$, it is possible to identify the action that maximizes $Q^{\pi_e}(x,a)$ (even if we cannot identify the value $Q^{\pi_e}(x,a)$). 
Finally, we remark that our worst-case regret bound is similar to that of \citet{ross2014reinforcement,sun2017deeply}, which can also outperform a suboptimal expert but require access to both expert's actions and reward signals---a much stronger query model than ours. \looseness=-1

\section{Discussion and Future Work} 

We presented interactive decision-making algorithms that learn from preference-based feedback while minimizing query complexity. Our algorithms for contextual bandits and imitation learning share worst-case regret bounds similar to the bounds of the state-of-art algorithms in standard settings while maintaining instance-dependent regret bounds and query complexity bounds. Notably, our imitation learning algorithm can outperform suboptimal experts, matching the result of \citep{ross2014reinforcement,sun2017deeply}, which operates under much stronger feedback. \looseness=-1

In terms of future work,  we believe our result on contextual dueling bandits can be extended to the stochastic setting where we may replace the eluder dimension with the value function disagreement coefficient \citep{foster2020instance}, which is typically smaller than the eluder dimension, and replace the online regression oracle by a supervised-learning batch regression oracle. We also conjecture that the dependence on the eluder dimension in the query complexity bound can be improved. Finally, another interesting direction is to develop practical implementations of our proposed algorithms. \looseness=-1

\subsection*{Acknowledgements} 
AS acknowledges support from the Simons Foundation and NSF through award DMS-2031883, as well as from the DOE through award DE-SC0022199. KS acknowledges support from NSF CAREER Award 1750575, and LinkedIn-Cornell grant.

\bibliography{references.bib}

\begin{thebibliography}{90}
\providecommand{\natexlab}[1]{#1}
\providecommand{\url}[1]{\texttt{#1}}
\expandafter\ifx\csname urlstyle\endcsname\relax
  \providecommand{\doi}[1]{doi: #1}\else
  \providecommand{\doi}{doi: \begingroup \urlstyle{rm}\Url}\fi

\bibitem[Abbasi-Yadkori et~al.(2011)Abbasi-Yadkori, P{\'a}l, and
  Szepesv{\'a}ri]{abbasi2011improved}
Yasin Abbasi-Yadkori, D{\'a}vid P{\'a}l, and Csaba Szepesv{\'a}ri.
\newblock Improved algorithms for linear stochastic bandits.
\newblock \emph{Advances in neural information processing systems}, 24, 2011.

\bibitem[Abbeel and Ng(2004)]{abbeel2004apprenticeship}
Pieter Abbeel and Andrew~Y Ng.
\newblock Apprenticeship learning via inverse reinforcement learning.
\newblock In \emph{Proceedings of the twenty-first international conference on
  Machine learning}, page~1, 2004.

\bibitem[Agarwal(2013)]{agarwal2013selective}
Alekh Agarwal.
\newblock Selective sampling algorithms for cost-sensitive multiclass
  prediction.
\newblock In \emph{International Conference on Machine Learning}, pages
  1220--1228. PMLR, 2013.

\bibitem[Agarwal et~al.(2019)Agarwal, Jiang, Kakade, and
  Sun]{agarwal2019reinforcement}
Alekh Agarwal, Nan Jiang, Sham~M Kakade, and Wen Sun.
\newblock Reinforcement learning: Theory and algorithms.
\newblock \emph{CS Dept., UW Seattle, Seattle, WA, USA, Tech. Rep}, pages
  10--4, 2019.

\bibitem[Ailon et~al.(2014)Ailon, Karnin, and Joachims]{ailon2014reducing}
Nir Ailon, Zohar Karnin, and Thorsten Joachims.
\newblock Reducing dueling bandits to cardinal bandits.
\newblock In \emph{International Conference on Machine Learning}, pages
  856--864. PMLR, 2014.

\bibitem[Audibert et~al.(2010)Audibert, Bubeck, and Munos]{audibert2010best}
Jean-Yves Audibert, S{\'e}bastien Bubeck, and R{\'e}mi Munos.
\newblock Best arm identification in multi-armed bandits.
\newblock In \emph{COLT}, pages 41--53, 2010.

\bibitem[Auer et~al.(2002)Auer, Cesa-Bianchi, Freund, and
  Schapire]{auer2002nonstochastic}
Peter Auer, Nicolo Cesa-Bianchi, Yoav Freund, and Robert~E Schapire.
\newblock The nonstochastic multiarmed bandit problem.
\newblock \emph{SIAM journal on computing}, 32\penalty0 (1):\penalty0 48--77,
  2002.

\bibitem[Ayoub et~al.(2020)Ayoub, Jia, Szepesvari, Wang, and
  Yang]{ayoub2020model}
Alex Ayoub, Zeyu Jia, Csaba Szepesvari, Mengdi Wang, and Lin Yang.
\newblock Model-based reinforcement learning with value-targeted regression.
\newblock In \emph{International Conference on Machine Learning}, pages
  463--474. PMLR, 2020.

\bibitem[Bengs et~al.(2021)Bengs, Busa-Fekete, El~Mesaoudi-Paul, and
  H{\"u}llermeier]{bengs2021preference}
Viktor Bengs, R{\'o}bert Busa-Fekete, Adil El~Mesaoudi-Paul, and Eyke
  H{\"u}llermeier.
\newblock Preference-based online learning with dueling bandits: A survey.
\newblock \emph{The Journal of Machine Learning Research}, 22\penalty0
  (1):\penalty0 278--385, 2021.

\bibitem[Biyik and Sadigh(2018)]{biyik2018batch}
Erdem Biyik and Dorsa Sadigh.
\newblock Batch active preference-based learning of reward functions.
\newblock In \emph{Conference on robot learning}, pages 519--528. PMLR, 2018.

\bibitem[Bradley and Terry(1952)]{bradley1952rank}
Ralph~Allan Bradley and Milton~E Terry.
\newblock Rank analysis of incomplete block designs: I. the method of paired
  comparisons.
\newblock \emph{Biometrika}, 39\penalty0 (3/4):\penalty0 324--345, 1952.

\bibitem[Brown et~al.(2019)Brown, Goo, Nagarajan, and
  Niekum]{brown2019extrapolating}
Daniel Brown, Wonjoon Goo, Prabhat Nagarajan, and Scott Niekum.
\newblock Extrapolating beyond suboptimal demonstrations via inverse
  reinforcement learning from observations.
\newblock In \emph{International conference on machine learning}, pages
  783--792. PMLR, 2019.

\bibitem[Brown et~al.(2020)Brown, Goo, and Niekum]{brown2020better}
Daniel~S Brown, Wonjoon Goo, and Scott Niekum.
\newblock Better-than-demonstrator imitation learning via automatically-ranked
  demonstrations.
\newblock In \emph{Conference on robot learning}, pages 330--359. PMLR, 2020.

\bibitem[Busa-Fekete et~al.(2014)Busa-Fekete, Sz{\"o}r{\'e}nyi, Weng, Cheng,
  and H{\"u}llermeier]{busa2014preference}
R{\'o}bert Busa-Fekete, Bal{\'a}zs Sz{\"o}r{\'e}nyi, Paul Weng, Weiwei Cheng,
  and Eyke H{\"u}llermeier.
\newblock Preference-based reinforcement learning: evolutionary direct policy
  search using a preference-based racing algorithm.
\newblock \emph{Machine learning}, 97:\penalty0 327--351, 2014.

\bibitem[Cesa-Bianchi and Lugosi(2006)]{cesa2006prediction}
Nicolo Cesa-Bianchi and G{\'a}bor Lugosi.
\newblock \emph{Prediction, learning, and games}.
\newblock Cambridge university press, 2006.

\bibitem[Cesa-Bianchi et~al.(2005)Cesa-Bianchi, Lugosi, and
  Stoltz]{cesa2005minimizing}
Nicolo Cesa-Bianchi, G{\'a}bor Lugosi, and Gilles Stoltz.
\newblock Minimizing regret with label efficient prediction.
\newblock \emph{IEEE Transactions on Information Theory}, 51\penalty0
  (6):\penalty0 2152--2162, 2005.

\bibitem[Chang et~al.(2015)Chang, Krishnamurthy, Agarwal, Daum{\'e}~III, and
  Langford]{chang2015learning}
Kai-Wei Chang, Akshay Krishnamurthy, Alekh Agarwal, Hal Daum{\'e}~III, and John
  Langford.
\newblock Learning to search better than your teacher.
\newblock In \emph{International Conference on Machine Learning}, pages
  2058--2066. PMLR, 2015.

\bibitem[Chen et~al.(2022)Chen, Zhong, Yang, Wang, and Wang]{chen2022human}
Xiaoyu Chen, Han Zhong, Zhuoran Yang, Zhaoran Wang, and Liwei Wang.
\newblock Human-in-the-loop: Provably efficient preference-based reinforcement
  learning with general function approximation.
\newblock In \emph{International Conference on Machine Learning}, pages
  3773--3793. PMLR, 2022.

\bibitem[Cheng and Boots(2018)]{cheng2018convergence}
Ching-An Cheng and Byron Boots.
\newblock Convergence of value aggregation for imitation learning.
\newblock In \emph{International Conference on Artificial Intelligence and
  Statistics}, pages 1801--1809. PMLR, 2018.

\bibitem[Christiano et~al.(2017)Christiano, Leike, Brown, Martic, Legg, and
  Amodei]{christiano2017deep}
Paul~F Christiano, Jan Leike, Tom Brown, Miljan Martic, Shane Legg, and Dario
  Amodei.
\newblock Deep reinforcement learning from human preferences.
\newblock \emph{Advances in neural information processing systems}, 30, 2017.

\bibitem[Chu and Ghahramani(2005)]{chu2005preference}
Wei Chu and Zoubin Ghahramani.
\newblock Preference learning with gaussian processes.
\newblock In \emph{Proceedings of the 22nd international conference on Machine
  learning}, pages 137--144, 2005.

\bibitem[Cohn et~al.(2011)Cohn, Durfee, and Singh]{cohn2011comparing}
Robert Cohn, Edmund Durfee, and Satinder Singh.
\newblock Comparing action-query strategies in semi-autonomous agents.
\newblock In \emph{Proceedings of the AAAI Conference on Artificial
  Intelligence}, volume~25, pages 1102--1107, 2011.

\bibitem[Dani et~al.(2008)Dani, Hayes, and Kakade]{dani2008stochastic}
Varsha Dani, Thomas~P Hayes, and Sham~M Kakade.
\newblock Stochastic linear optimization under bandit feedback.
\newblock In \emph{21st Annual Conference on Learning Theory}, pages 355--366,
  2008.

\bibitem[Daum{\'e} et~al.(2009)Daum{\'e}, Langford, and Marcu]{daume2009search}
Hal Daum{\'e}, John Langford, and Daniel Marcu.
\newblock Search-based structured prediction.
\newblock \emph{Machine learning}, 75:\penalty0 297--325, 2009.

\bibitem[Dekel et~al.(2012)Dekel, Gentile, and Sridharan]{dekel2012selective}
Ofer Dekel, Claudio Gentile, and Karthik Sridharan.
\newblock Selective sampling and active learning from single and multiple
  {{experts}}.
\newblock \emph{The Journal of Machine Learning Research}, 13\penalty0
  (1):\penalty0 2655--2697, 2012.

\bibitem[Du et~al.(2019)Du, Luo, Wang, and Zhang]{du2019provably}
Simon~S Du, Yuping Luo, Ruosong Wang, and Hanrui Zhang.
\newblock Provably efficient q-learning with function approximation via
  distribution shift error checking oracle.
\newblock \emph{Advances in Neural Information Processing Systems}, 32, 2019.

\bibitem[Dud{\'\i}k et~al.(2015)Dud{\'\i}k, Hofmann, Schapire, Slivkins, and
  Zoghi]{dudik2015contextual}
Miroslav Dud{\'\i}k, Katja Hofmann, Robert~E Schapire, Aleksandrs Slivkins, and
  Masrour Zoghi.
\newblock Contextual dueling bandits.
\newblock In \emph{Conference on Learning Theory}, pages 563--587. PMLR, 2015.

\bibitem[Foster and Rakhlin(2020)]{foster2020beyond}
Dylan Foster and Alexander Rakhlin.
\newblock Beyond ucb: Optimal and efficient contextual bandits with regression
  oracles.
\newblock In \emph{International Conference on Machine Learning}, pages
  3199--3210. PMLR, 2020.

\bibitem[Foster et~al.(2018{\natexlab{a}})Foster, Agarwal, Dud{\'\i}k, Luo, and
  Schapire]{foster2018practical}
Dylan Foster, Alekh Agarwal, Miroslav Dud{\'\i}k, Haipeng Luo, and Robert
  Schapire.
\newblock Practical contextual bandits with regression oracles.
\newblock In \emph{International Conference on Machine Learning}, pages
  1539--1548. PMLR, 2018{\natexlab{a}}.

\bibitem[Foster et~al.(2021)Foster, Rakhlin, Simchi-Levi, and
  Xu]{foster2020instance}
Dylan Foster, Alexander Rakhlin, David Simchi-Levi, and Yunzong Xu.
\newblock Instance-dependent complexity of contextual bandits and reinforcement
  learning: A disagreement-based perspective.
\newblock In \emph{Conference on Learning Theory}, pages 2059--2059. PMLR,
  2021.

\bibitem[Foster et~al.(2018{\natexlab{b}})Foster, Kale, Luo, Mohri, and
  Sridharan]{foster2018logistic}
Dylan~J Foster, Satyen Kale, Haipeng Luo, Mehryar Mohri, and Karthik Sridharan.
\newblock Logistic regression: The importance of being improper.
\newblock In \emph{Conference On Learning Theory}, pages 167--208. PMLR,
  2018{\natexlab{b}}.

\bibitem[F{\"u}rnkranz et~al.(2012)F{\"u}rnkranz, H{\"u}llermeier, Cheng, and
  Park]{furnkranz2012preference}
Johannes F{\"u}rnkranz, Eyke H{\"u}llermeier, Weiwei Cheng, and Sang-Hyeun
  Park.
\newblock Preference-based reinforcement learning: a formal framework and a
  policy iteration algorithm.
\newblock \emph{Machine learning}, 89:\penalty0 123--156, 2012.

\bibitem[Garivier et~al.(2019)Garivier, M{\'e}nard, and
  Stoltz]{garivier2019explore}
Aur{\'e}lien Garivier, Pierre M{\'e}nard, and Gilles Stoltz.
\newblock Explore first, exploit next: The true shape of regret in bandit
  problems.
\newblock \emph{Mathematics of Operations Research}, 44\penalty0 (2):\penalty0
  377--399, 2019.

\bibitem[Hanneke and Yang(2015)]{hanneke2015minimax}
Steve Hanneke and Liu Yang.
\newblock Minimax analysis of active learning.
\newblock \emph{J. Mach. Learn. Res.}, 16\penalty0 (1):\penalty0 3487--3602,
  2015.

\bibitem[Hanneke and Yang(2021)]{hanneke2021toward}
Steve Hanneke and Liu Yang.
\newblock Toward a general theory of online selective sampling: Trading off
  mistakes and queries.
\newblock In \emph{International Conference on Artificial Intelligence and
  Statistics}, pages 3997--4005. PMLR, 2021.

\bibitem[He et~al.(2021)He, Zhou, and Gu]{he2021logarithmic}
Jiafan He, Dongruo Zhou, and Quanquan Gu.
\newblock Logarithmic regret for reinforcement learning with linear function
  approximation.
\newblock In \emph{International Conference on Machine Learning}, pages
  4171--4180. PMLR, 2021.

\bibitem[Huang et~al.(2022)Huang, Lee, Wang, and Yang]{huang2021towards}
Baihe Huang, Jason~D Lee, Zhaoran Wang, and Zhuoran Yang.
\newblock Towards general function approximation in zero-sum markov games.
\newblock In \emph{International Conference on Learning Representations}, 2022.

\bibitem[Ishfaq et~al.(2021)Ishfaq, Cui, Nguyen, Ayoub, Yang, Wang, Precup, and
  Yang]{ishfaq2021randomized}
Haque Ishfaq, Qiwen Cui, Viet Nguyen, Alex Ayoub, Zhuoran Yang, Zhaoran Wang,
  Doina Precup, and Lin Yang.
\newblock Randomized exploration in reinforcement learning with general value
  function approximation.
\newblock In \emph{International Conference on Machine Learning}, pages
  4607--4616. PMLR, 2021.

\bibitem[Jain et~al.(2015)Jain, Sharma, Joachims, and Saxena]{jain2015learning}
Ashesh Jain, Shikhar Sharma, Thorsten Joachims, and Ashutosh Saxena.
\newblock Learning preferences for manipulation tasks from online coactive
  feedback.
\newblock \emph{The International Journal of Robotics Research}, 34\penalty0
  (10):\penalty0 1296--1313, 2015.

\bibitem[Jin and Luo(2020)]{jin2020simultaneously}
Tiancheng Jin and Haipeng Luo.
\newblock Simultaneously learning stochastic and adversarial episodic mdps with
  known transition.
\newblock \emph{Advances in neural information processing systems},
  33:\penalty0 16557--16566, 2020.

\bibitem[Kakade and Tewari(2008)]{kakade2008generalization}
Sham~M Kakade and Ambuj Tewari.
\newblock On the generalization ability of online strongly convex programming
  algorithms.
\newblock \emph{Advances in Neural Information Processing Systems}, 21, 2008.

\bibitem[Komiyama et~al.(2015)Komiyama, Honda, Kashima, and
  Nakagawa]{komiyama2015regret}
Junpei Komiyama, Junya Honda, Hisashi Kashima, and Hiroshi Nakagawa.
\newblock Regret lower bound and optimal algorithm in dueling bandit problem.
\newblock In \emph{Conference on learning theory}, pages 1141--1154. PMLR,
  2015.

\bibitem[Krishnamurthy et~al.(2017)Krishnamurthy, Agarwal, Huang,
  Daum{\'e}~III, and Langford]{krishnamurthy2017active}
Akshay Krishnamurthy, Alekh Agarwal, Tzu-Kuo Huang, Hal Daum{\'e}~III, and John
  Langford.
\newblock Active learning for cost-sensitive classification.
\newblock In \emph{International Conference on Machine Learning}, pages
  1915--1924. PMLR, 2017.

\bibitem[Langford and Zhang(2007)]{langford2007epoch}
John Langford and Tong Zhang.
\newblock The epoch-greedy algorithm for multi-armed bandits with side
  information.
\newblock \emph{Advances in neural information processing systems}, 20, 2007.

\bibitem[Laskey et~al.(2016)Laskey, Staszak, Hsieh, Mahler, Pokorny, Dragan,
  and Goldberg]{laskey2016shiv}
Michael Laskey, Sam Staszak, Wesley Yu-Shu Hsieh, Jeffrey Mahler, Florian~T
  Pokorny, Anca~D Dragan, and Ken Goldberg.
\newblock Shiv: Reducing supervisor burden in dagger using support vectors for
  efficient learning from demonstrations in high dimensional state spaces.
\newblock In \emph{2016 IEEE International Conference on Robotics and
  Automation (ICRA)}, pages 462--469. IEEE, 2016.

\bibitem[Lattimore and Szepesv{\'a}ri(2020)]{lattimore2020bandit}
Tor Lattimore and Csaba Szepesv{\'a}ri.
\newblock \emph{Bandit algorithms}.
\newblock Cambridge University Press, 2020.

\bibitem[Lee et~al.(2021{\natexlab{a}})Lee, Smith, Dragan, and
  Abbeel]{lee2021b}
Kimin Lee, Laura Smith, Anca Dragan, and Pieter Abbeel.
\newblock B-pref: Benchmarking preference-based reinforcement learning.
\newblock In \emph{Thirty-fifth Conference on Neural Information Processing
  Systems Datasets and Benchmarks Track (Round 1)}, 2021{\natexlab{a}}.

\bibitem[Lee et~al.(2021{\natexlab{b}})Lee, Smith, and Abbeel]{lee2021pebble}
Kimin Lee, Laura~M Smith, and Pieter Abbeel.
\newblock Pebble: Feedback-efficient interactive reinforcement learning via
  relabeling experience and unsupervised pre-training.
\newblock In \emph{International Conference on Machine Learning}, pages
  6152--6163. PMLR, 2021{\natexlab{b}}.

\bibitem[Li et~al.(2023)Li, Yang, and Wang]{li2023reinforcement}
Zihao Li, Zhuoran Yang, and Mengdi Wang.
\newblock Reinforcement learning with human feedback: Learning dynamic choices
  via pessimism.
\newblock \emph{arXiv preprint arXiv:2305.18438}, 2023.

\bibitem[Lightman et~al.(2023)Lightman, Kosaraju, Burda, Edwards, Baker, Lee,
  Leike, Schulman, Sutskever, and Cobbe]{lightman2023let}
Hunter Lightman, Vineet Kosaraju, Yura Burda, Harri Edwards, Bowen Baker, Teddy
  Lee, Jan Leike, John Schulman, Ilya Sutskever, and Karl Cobbe.
\newblock Let's verify step by step.
\newblock \emph{arXiv preprint arXiv:2305.20050}, 2023.

\bibitem[Lykouris et~al.(2021)Lykouris, Simchowitz, Slivkins, and
  Sun]{lykouris2021corruption}
Thodoris Lykouris, Max Simchowitz, Alex Slivkins, and Wen Sun.
\newblock Corruption-robust exploration in episodic reinforcement learning.
\newblock In \emph{Conference on Learning Theory}, pages 3242--3245. PMLR,
  2021.

\bibitem[Myers et~al.(2023)Myers, B{\i}y{\i}k, and Sadigh]{myers2023active}
Vivek Myers, Erdem B{\i}y{\i}k, and Dorsa Sadigh.
\newblock Active reward learning from online preferences.
\newblock \emph{arXiv preprint arXiv:2302.13507}, 2023.

\bibitem[Novoseller et~al.(2020)Novoseller, Wei, Sui, Yue, and
  Burdick]{novoseller2020dueling}
Ellen Novoseller, Yibing Wei, Yanan Sui, Yisong Yue, and Joel Burdick.
\newblock Dueling posterior sampling for preference-based reinforcement
  learning.
\newblock In \emph{Conference on Uncertainty in Artificial Intelligence}, pages
  1029--1038. PMLR, 2020.

\bibitem[OpenAI(2023)]{openai2023gpt}
OpenAI.
\newblock Gpt-4 technical report.
\newblock \emph{arXiv}, 2023.

\bibitem[Osa et~al.(2018)Osa, Pajarinen, Neumann, Bagnell, Abbeel, Peters,
  et~al.]{osa2018algorithmic}
Takayuki Osa, Joni Pajarinen, Gerhard Neumann, J~Andrew Bagnell, Pieter Abbeel,
  Jan Peters, et~al.
\newblock An algorithmic perspective on imitation learning.
\newblock \emph{Foundations and Trends{\textregistered} in Robotics},
  7\penalty0 (1-2):\penalty0 1--179, 2018.

\bibitem[Osband and Van~Roy(2014)]{osband2014model}
Ian Osband and Benjamin Van~Roy.
\newblock Model-based reinforcement learning and the eluder dimension.
\newblock \emph{Advances in Neural Information Processing Systems}, 27, 2014.

\bibitem[Ouyang et~al.(2022)Ouyang, Wu, Jiang, Almeida, Wainwright, Mishkin,
  Zhang, Agarwal, Slama, Ray, et~al.]{ouyang2022training}
Long Ouyang, Jeffrey Wu, Xu~Jiang, Diogo Almeida, Carroll Wainwright, Pamela
  Mishkin, Chong Zhang, Sandhini Agarwal, Katarina Slama, Alex Ray, et~al.
\newblock Training language models to follow instructions with human feedback.
\newblock \emph{Advances in Neural Information Processing Systems},
  35:\penalty0 27730--27744, 2022.

\bibitem[Pacchiano et~al.(2021)Pacchiano, Saha, and Lee]{pacchiano2021dueling}
Aldo Pacchiano, Aadirupa Saha, and Jonathan Lee.
\newblock Dueling rl: reinforcement learning with trajectory preferences.
\newblock \emph{arXiv preprint arXiv:2111.04850}, 2021.

\bibitem[Rakhlin and Sridharan(2014)]{rakhlin2014online}
Alexander Rakhlin and Karthik Sridharan.
\newblock Online non-parametric regression.
\newblock In \emph{Conference on Learning Theory}, pages 1232--1264. PMLR,
  2014.

\bibitem[Ross and Bagnell(2014)]{ross2014reinforcement}
Stephane Ross and J~Andrew Bagnell.
\newblock Reinforcement and imitation learning via interactive no-regret
  learning.
\newblock \emph{arXiv preprint arXiv:1406.5979}, 2014.

\bibitem[Ross et~al.(2011)Ross, Gordon, and Bagnell]{ross2011reduction}
St{\'e}phane Ross, Geoffrey Gordon, and Drew Bagnell.
\newblock A reduction of imitation learning and structured prediction to
  no-regret online learning.
\newblock In \emph{Proceedings of the fourteenth international conference on
  artificial intelligence and statistics}, pages 627--635. JMLR Workshop and
  Conference Proceedings, 2011.

\bibitem[Ross et~al.(2013)Ross, Melik-Barkhudarov, Shankar, Wendel, Dey,
  Bagnell, and Hebert]{ross2013learning}
St{\'e}phane Ross, Narek Melik-Barkhudarov, Kumar~Shaurya Shankar, Andreas
  Wendel, Debadeepta Dey, J~Andrew Bagnell, and Martial Hebert.
\newblock Learning monocular reactive uav control in cluttered natural
  environments.
\newblock In \emph{2013 IEEE international conference on robotics and
  automation}, pages 1765--1772. IEEE, 2013.

\bibitem[Russo and Van~Roy(2013)]{russo2013eluder}
Daniel Russo and Benjamin Van~Roy.
\newblock Eluder dimension and the sample complexity of optimistic exploration.
\newblock \emph{Advances in Neural Information Processing Systems}, 26, 2013.

\bibitem[Sadigh et~al.(2017)Sadigh, Dragan, Sastry, and
  Seshia]{sadigh2017active}
Dorsa Sadigh, Anca~D Dragan, Shankar Sastry, and Sanjit~A Seshia.
\newblock \emph{Active preference-based learning of reward functions}.
\newblock 2017.

\bibitem[Saha and Gaillard(2021)]{saha2021dueling}
Aadirupa Saha and Pierre Gaillard.
\newblock Dueling bandits with adversarial sleeping.
\newblock \emph{Advances in Neural Information Processing Systems},
  34:\penalty0 27761--27771, 2021.

\bibitem[Saha and Gaillard(2022)]{saha2022versatile}
Aadirupa Saha and Pierre Gaillard.
\newblock Versatile dueling bandits: Best-of-both world analyses for learning
  from relative preferences.
\newblock In \emph{International Conference on Machine Learning}, pages
  19011--19026. PMLR, 2022.

\bibitem[Saha and Krishnamurthy(2022)]{saha2022efficient}
Aadirupa Saha and Akshay Krishnamurthy.
\newblock Efficient and optimal algorithms for contextual dueling bandits under
  realizability.
\newblock In \emph{International Conference on Algorithmic Learning Theory},
  pages 968--994. PMLR, 2022.

\bibitem[Saha et~al.(2023)Saha, Pacchiano, and Lee]{saha2023dueling}
Aadirupa Saha, Aldo Pacchiano, and Jonathan Lee.
\newblock Dueling rl: Reinforcement learning with trajectory preferences.
\newblock In \emph{International Conference on Artificial Intelligence and
  Statistics}, pages 6263--6289. PMLR, 2023.

\bibitem[Sekhari et~al.(2023)Sekhari, Sridharan, Sun, and
  Wu]{sekhari2023selective}
Ayush Sekhari, Karthik Sridharan, Wen Sun, and Runzhe Wu.
\newblock Selective sampling and imitation learning via online regression.
\newblock \emph{arXiv preprint arXiv:2307.04998}, 2023.

\bibitem[Simchi-Levi and Xu(2022)]{simchi2022bypassing}
David Simchi-Levi and Yunzong Xu.
\newblock Bypassing the monster: A faster and simpler optimal algorithm for
  contextual bandits under realizability.
\newblock \emph{Mathematics of Operations Research}, 47\penalty0 (3):\penalty0
  1904--1931, 2022.

\bibitem[Simchowitz and Jamieson(2019)]{simchowitz2019non}
Max Simchowitz and Kevin~G Jamieson.
\newblock Non-asymptotic gap-dependent regret bounds for tabular mdps.
\newblock \emph{Advances in Neural Information Processing Systems}, 32, 2019.

\bibitem[Stiennon et~al.(2020)Stiennon, Ouyang, Wu, Ziegler, Lowe, Voss,
  Radford, Amodei, and Christiano]{stiennon2020learning}
Nisan Stiennon, Long Ouyang, Jeffrey Wu, Daniel Ziegler, Ryan Lowe, Chelsea
  Voss, Alec Radford, Dario Amodei, and Paul~F Christiano.
\newblock Learning to summarize with human feedback.
\newblock \emph{Advances in Neural Information Processing Systems},
  33:\penalty0 3008--3021, 2020.

\bibitem[Sun et~al.(2017)Sun, Venkatraman, Gordon, Boots, and
  Bagnell]{sun2017deeply}
Wen Sun, Arun Venkatraman, Geoffrey~J Gordon, Byron Boots, and J~Andrew
  Bagnell.
\newblock Deeply aggrevated: Differentiable imitation learning for sequential
  prediction.
\newblock In \emph{International conference on machine learning}, pages
  3309--3318. PMLR, 2017.

\bibitem[Taranovic et~al.(2022)Taranovic, Kupcsik, Freymuth, and
  Neumann]{taranovic2022adversarial}
Aleksandar Taranovic, Andras~Gabor Kupcsik, Niklas Freymuth, and Gerhard
  Neumann.
\newblock Adversarial imitation learning with preferences.
\newblock In \emph{The Eleventh International Conference on Learning
  Representations}, 2022.

\bibitem[Tsybakov(2004)]{tsybakov2004optimal}
Alexander~B Tsybakov.
\newblock Optimal aggregation of classifiers in statistical learning.
\newblock \emph{The Annals of Statistics}, 32\penalty0 (1):\penalty0 135--166,
  2004.

\bibitem[Wang et~al.(2020)Wang, Salakhutdinov, and Yang]{wang2020reinforcement}
Ruosong Wang, Russ~R Salakhutdinov, and Lin Yang.
\newblock Reinforcement learning with general value function approximation:
  Provably efficient approach via bounded eluder dimension.
\newblock \emph{Advances in Neural Information Processing Systems},
  33:\penalty0 6123--6135, 2020.

\bibitem[Wen and Van~Roy(2013)]{wen2013efficient}
Zheng Wen and Benjamin Van~Roy.
\newblock Efficient exploration and value function generalization in
  deterministic systems.
\newblock \emph{Advances in Neural Information Processing Systems}, 26, 2013.

\bibitem[Wirth and F{\"u}rnkranz(2014)]{wirth2014learning}
Christian Wirth and Johannes F{\"u}rnkranz.
\newblock On learning from game annotations.
\newblock \emph{IEEE Transactions on Computational Intelligence and AI in
  Games}, 7\penalty0 (3):\penalty0 304--316, 2014.

\bibitem[Wirth et~al.(2017)Wirth, Akrour, Neumann, F{\"u}rnkranz,
  et~al.]{wirth2017survey}
Christian Wirth, Riad Akrour, Gerhard Neumann, Johannes F{\"u}rnkranz, et~al.
\newblock A survey of preference-based reinforcement learning methods.
\newblock \emph{Journal of Machine Learning Research}, 18\penalty0
  (136):\penalty0 1--46, 2017.

\bibitem[Wu and Liu(2016)]{wu2016double}
Huasen Wu and Xin Liu.
\newblock Double thompson sampling for dueling bandits.
\newblock \emph{Advances in neural information processing systems}, 29, 2016.

\bibitem[Wu et~al.(2023)Wu, Jin, Lou, Farnoud, and Gu]{wu2023borda}
Yue Wu, Tao Jin, Hao Lou, Farzad Farnoud, and Quanquan Gu.
\newblock Borda regret minimization for generalized linear dueling bandits.
\newblock \emph{arXiv preprint arXiv:2303.08816}, 2023.

\bibitem[Xu et~al.(2020)Xu, Wang, Yang, Singh, and Dubrawski]{xu2020preference}
Yichong Xu, Ruosong Wang, Lin Yang, Aarti Singh, and Artur Dubrawski.
\newblock Preference-based reinforcement learning with finite-time guarantees.
\newblock \emph{Advances in Neural Information Processing Systems},
  33:\penalty0 18784--18794, 2020.

\bibitem[Yue and Joachims(2011)]{yue2011beat}
Yisong Yue and Thorsten Joachims.
\newblock Beat the mean bandit.
\newblock In \emph{Proceedings of the 28th international conference on machine
  learning (ICML-11)}, pages 241--248. Citeseer, 2011.

\bibitem[Yue et~al.(2012)Yue, Broder, Kleinberg, and Joachims]{yue2012k}
Yisong Yue, Josef Broder, Robert Kleinberg, and Thorsten Joachims.
\newblock The k-armed dueling bandits problem.
\newblock \emph{Journal of Computer and System Sciences}, 78\penalty0
  (5):\penalty0 1538--1556, 2012.

\bibitem[Zhan et~al.(2023)Zhan, Uehara, Sun, and Lee]{zhan2023query}
Wenhao Zhan, Masatoshi Uehara, Wen Sun, and Jason~D Lee.
\newblock How to query human feedback efficiently in rl?
\newblock \emph{arXiv preprint arXiv:2305.18505}, 2023.

\bibitem[Zhang et~al.(2022)Zhang, Carroll, Bobu, and Dragan]{zhang2022time}
David Zhang, Micah Carroll, Andreea Bobu, and Anca Dragan.
\newblock Time-efficient reward learning via visually assisted cluster ranking.
\newblock \emph{arXiv preprint arXiv:2212.00169}, 2022.

\bibitem[Zhu et~al.(2023)Zhu, Jiao, and Jordan]{zhu2023principled}
Banghua Zhu, Jiantao Jiao, and Michael~I Jordan.
\newblock Principled reinforcement learning with human feedback from pairwise
  or $ k $-wise comparisons.
\newblock In \emph{International Conference on Machine Learning}. PMLR, 2023.

\bibitem[Zhu and Nowak(2022)]{zhu2022efficient}
Yinglun Zhu and Robert Nowak.
\newblock Efficient active learning with abstention.
\newblock \emph{Advances in Neural Information Processing Systems},
  35:\penalty0 35379--35391, 2022.

\bibitem[Ziebart et~al.(2008)Ziebart, Maas, Bagnell, and
  Dey]{ziebart2008maximum}
Brian~D Ziebart, Andrew Maas, J~Andrew Bagnell, and Anind~K Dey.
\newblock Maximum entropy inverse reinforcement learning.
\newblock In \emph{Proceedings of the 23rd national conference on Artificial
  intelligence-Volume 3}, pages 1433--1438, 2008.

\bibitem[Zoghi et~al.(2014)Zoghi, Whiteson, Munos, and
  Rijke]{zoghi2014relative}
Masrour Zoghi, Shimon Whiteson, Remi Munos, and Maarten Rijke.
\newblock Relative upper confidence bound for the k-armed dueling bandit
  problem.
\newblock In \emph{International conference on machine learning}, pages 10--18.
  PMLR, 2014.

\end{thebibliography}
\bibliographystyle{plainnat}

\newpage
\appendix

\section{Preliminaries}

\begin{lemma}[{\citet[Lemma 3]{kakade2008generalization}}]\label{lem:free2}
Suppose $X_1, \ldots, X_T$ is a martingale difference sequence with $\left|X_t\right| \leq b$. Let
$$
\operatorname{Var}_t X_t=\operatorname{Var}\left(X_t \mid X_1, \ldots, X_{t-1}\right)
$$
Let $V=\sum_{t=1}^T \operatorname{Var}_t X_t$ be the sum of conditional variances of $X_t$ 's. Further, let $\sigma=\sqrt{V}$. Then we have, for any $\delta<1 / e$ and $T \geq 3$,
$$
\Pr\left(\sum_{t=1}^T X_t>\max \{2 \sigma, 3 b \sqrt{\ln (1 / \delta)}\} \sqrt{\ln (1 / \delta)}\right) \leq 4 \ln (T) \delta.
$$	
\end{lemma}

\begin{lemma}[{\citet[Lemma 3]{foster2020beyond}}]\label{lem:igw-general}
	For any vector $\^y\in[0,1]^A$, if we define $p$ to be
	\begin{align*}
		p(a)
=		\begin{cases}
			\frac{1}{A+\gamma \big(\^y(\^a)-\^y(a)\big)} & \text{if } a\not=\^a,\\
			1-\sum_{a\not=\^a}p(a) & \text{if } a=\^a
		\end{cases}
	\end{align*}
	where $\^a=\argmax_a \^y(a)$, then for any $y^\star\in[0,1]^A$ and $\gamma>0$, we have
	$$
		\E_{a\sim p}\left[\Big(y^\star(a^\star)-y^\star(a)\Big)-\gamma\Big(\^y(a)-y^\star(a)\Big)^2\right]\leq \frac{A}{\gamma}.
	$$
\end{lemma}

\begin{lemma}[{\citet[Lemma 2]{zhu2022efficient}}]\label{lem:freedman}
Let $(Z_t)_{t \leq T}$ to be real-valued sequence of positive random variables adapted to a filtration $\mathfrak{F}_t$. If $\left|Z_t\right| \leq B$ almost surely, then with probability at least $1-\delta$,
$$
\sum_{t=1}^T Z_t \leq \frac{3}{2} \sum_{t=1}^T \mathbb{E}_t\left[Z_t\right]+4 B \log \left(2 \delta^{-1}\right),
$$
and
$$
\sum_{t=1}^T \mathbb{E}_t\left[Z_t\right] \leq 2 \sum_{t=1}^T Z_t+8 B \log \left(2 \delta^{-1}\right).
$$
\end{lemma}

\begin{lemma}[Performance difference lemma \citep{agarwal2019reinforcement}]\label{lem:pdl}
	For any two policies $\pi$ and $\pi'$ and any state $x_0\in\+X$, we have
	\begin{align*}
		V^\pi_0(x_0)-V^{\pi'}_0(x_0)
		=
		\sum_{h=0}^{H-1} \E_{x_h,a_h\sim d^\pi_{x_0,h}} \big[A^{\pi'}_h(x_h,a_h)\big]
	\end{align*}
	where $A^\pi_h(x,a)=Q^\pi_h(x,a)-V^\pi_h(x,a)$ and $d^\pi_{x_0,h}(x,a)$ is the probability of $\pi$ reaching the state-action pair $(x,a)$ at time step $h$ starting from initial state $x_0$.
\end{lemma}

\begin{lemma}\label{lem:kl-bern}
	For any two Bernoulli distributions $\@{Bern}(x)$ and $\@{Bern}(y)$ with $x,y\in[b,1-b]$ for some $0<b\leq 1/2$, the KL divergence is bounded as
	\begin{align*}
		\@{KL}\Big(\@{Bern}(x),\@{Bern}(y)\Big)\leq \frac{2(x-y)^2}{b}.
	\end{align*}
\end{lemma}
\begin{proof}[Proof of \Cref{lem:kl-bern}]
Denote $\Delta=x-y$. Then, by definition, we have
	\begin{align*}
		\@{KL}\Big(\@{Bern}(x),\@{Bern}(y)\Big)
		= & x\ln \frac{x}{y} + (1-x)\ln \frac{1-x}{1-y}\\
		= & x\ln \frac{x}{x-\Delta} + (1-x)\ln \frac{1-x}{1-x+\Delta}\\
		= & x\ln \left(1+\frac{\Delta}{x-\Delta}\right) + (1-x)\ln\left(1- \frac{\Delta}{1-x+\Delta}\right)
	\end{align*}
	Since $\ln(1+x)\leq x$ for all $x>-1$, we have
	\begin{align*}
		\@{KL}\Big(\@{Bern}(x),\@{Bern}(y)\Big)
		\leq & x\cdot \frac{\Delta}{x-\Delta} - (1-x)\cdot \frac{\Delta}{1-x+\Delta}\\
		= & \Delta\cdot\left( \frac{x}{x-\Delta} - \frac{1-x}{1-x+\Delta}\right)\\
		= & \Delta\cdot\left(\frac{\Delta}{x-\Delta} + \frac{\Delta}{1-x+\Delta}\right)\\
		\leq & \Delta^2\cdot\left(\frac{1}{y} + \frac{1}{1-y}\right)
		\leq \frac{2\Delta^2}{b}.
	\end{align*}
\end{proof}

\section{Missing Proofs}\label{sec:missing-pf}

\subsection{Supporting Lemmas}

\begin{definition}[Strong convexity]\label{def:strong-conv}
	A function $\Phi:[-1,1]\rightarrow\=R$ is $\alpha$-strongly-convex if for all $u,u'\in\=R$, we have
	\begin{align*}
		\frac{\alpha}{2}(u'-u)^2\leq \Phi(u') - \Phi(u) - \nabla\Phi(u)(u'-u).
	\end{align*}
	where $\nabla\Phi$ means the derivative of $\Phi$.
\end{definition}

\begin{lemma}\label{lem:width-lower-bound}
	For any $t\in[T]$, if $f^\star\in\+F_t$, then we have $w_t\geq\Delta$ whenever $|\+A_t|>1$.
\end{lemma}
\begin{proof}[Proof of \cref{lem:width-lower-bound}]
	When $|\+A_t|>1$, we know there exists a function $f'\in\+F_t$ satisfying
	\begin{align*}
		a':=\pi_{f'}(x_t)\neq \pi_{f^\star}(x_t)=:a^\star_t.
	\end{align*}
	Then we have $\Delta
	\leq f^\star(x_t,a^\star_t,a')\leq f^\star(x_t,a^\star_t,a')-f'(x_t,a^\star_t,a')\leq w_t$ where the second inequality holds since $f'(x_t,a^\star_t,a')\leq0$.
\end{proof}

\begin{lemma}\label{lem:regret-bounded-by-w}
	For any $t\in[T]$ and any arm $a\in\+A_t$, we have $f^\star(x_t,\pi_{f^\star}(x_t),a)\leq w_t$.
\end{lemma}
\begin{proof}[Proof of \Cref{lem:regret-bounded-by-w}]
For any $a\in\+A_t$, by the definition of $\+A_t$, there must exists a function $f$ for which $a=\pi_{f}(x_t)$. Hence,
	\begin{align*}
		f^\star(x_t,\pi_{f^\star}(x_t),a)
		\leq
		f^\star(x_t,\pi_{f^\star}(x_t),a)-f(x_t,\pi_{f^\star}(x_t),a)\leq w_t,
	\end{align*}
	where the first inequality holds since $f(x_t,\pi_{f^\star}(x_t),a)\leq0$.
\end{proof}

The following lemma is adapted from \citet[Lemma 2]{agarwal2013selective}.
\begin{lemma}\label{lem:pointwise-bound}
	The following holds with probability at least $1-\delta$ for any $T>3$,
	\begin{align*}
	\sum_{t=1}^TZ_t\big(f^\star(x_t,a_t,b_t)-f_t(x_t,a_t,b_t)\big)^2
	\leq\frac{4\Upsilon}{\alpha}+\frac{16+24\alpha}{\alpha^2}\log\big(4\delta^{-1}\log(T)\big).
	\end{align*} 
\end{lemma}

\begin{proof}[Proof of \Cref{lem:pointwise-bound}]
Throughout the proof, we denote $z_t:=(x_t,a_t,b_t)$ for notational simplicity. We define $D_\Phi$ as the Bregman divergence of the function $\Phi$:
\begin{align*}
	D_\Phi(u,v)=\Phi(u)-\Phi(v)-\phi(v)(u-v)
\end{align*}
where we recall that $\phi=\Phi'$ is the derivative of $\Phi$. Since $\Phi$ is $\alpha$-strong convex, we have $\alpha(u-v)^2/2\leq D_\Phi(u,v)$, and hence,
\begin{align}\label{eq:square-bregman}
	\sum_{t=1}^T Z_t\big(f^\star(z_t)-f_t(z_t)\big)^2\leq\frac{2}{\alpha}\sum_{t=1}^T Z_tD_\Phi(f_t(z_t),f^\star(z_t)).
\end{align}
Hence, it suffice to derive an upper bound for the Bregman divergence in the right hand side above. Define $\nu_t$ as below:
\begin{align*} 
	\nu_t :=&Z_t\Big[D_\Phi\left(f_t(z_t), f^\star(z_t)\right)-\left(\ell_\phi\left(f_t(z_t),y_t\right)-\ell_\phi\left(f^\star(z_t), y_t\right)\right)\Big] \\
	 =&Z_t\Big[D_\Phi\left(f_t(z_t), f^\star(z_t)\right)-\left(\Phi\left(f_t(z_t)\right)-(y_t+1)f_t(z_t)/2-\Phi\left(f^\star(z_t)\right)+(y_t+1)f^\star(z_t)/2\right)\Big] \\
	 =&Z_t\Big[\Phi\left(f_t(z_t)\right)-\Phi\left(f^\star(z_t)\right)-\phi\left(f^\star(z_t)\right)\left(f_t(z_t)-f^\star(z_t)\right)\\
    &\quad-\left(\Phi\left(f_t(z_t)\right)-(y_t+1)f_t(z_t)/2-\Phi\left(f^\star(z_t)\right)+(y_t+1)f^\star(z_t)/2\right)\Big] \\
	 =&Z_t\big(f_t(z_t)-f^\star(z_t)\big)\big((y_t+1)/2-\phi(f^\star(z_t))\big)
	\end{align*}
We note that $\E_t[(y_t+1)/2]=\phi(f^\star(z_t))$, and thus $\E_t[\nu_t]=0$, which means $\nu_t$ is a martingale difference sequence. Now we bound the value and the conditional variance of $\nu_t$ in order to derive concentration results.
\begin{enumerate}
	\item Bound the value of $\nu_t$: 
	\begin{align*}
		|\nu_t|\leq| (y_t+1)/2-\phi\left(f^\star(z_t)\right)|\cdot|f_t(z_t)-f^\star(z_t)|\leq1\cdot2 =2.
	\end{align*}
	\item Bound the conditional variance of $\nu_t$: 
	\begin{align*}
		\E_t[\nu_t^2]
		=& Z_t\E_t\left[\left( (y_t+1)/2-\phi\left(f^\star(z_t)\right)\right)^2\left(f_t(z_t)-f^\star(z_t)\right)^2\right]\\
		\leq& Z_t\E_t\left[\left(f_t(z_t)-f^\star(z_t)\right)^2\right]\\
		\leq& Z_t\E_t\left[\frac{2}{\alpha}\cdot D_\Phi(f_t(z_t),f^\star(z_t))\right]\\
		\leq& \frac{2Z_t}{\alpha}D_\Phi(f_t(z_t),f^\star(z_t))
	\end{align*}
	where for the last line we note that $x_t,g_t$ are measurable at $t$.
\end{enumerate}
Now we apply \Cref{lem:free2}, which yields for any $\delta<1/e$ and $T>3$, with probability at least $1-4\delta\log(T)$,
\begin{align*}
	\sum_{t=1}^T\nu_t
	\leq&\max\left\{
	2\sqrt{\sum_{t=1}^T\frac{2Z_t}{\alpha}D_\Phi(f_t(z_t),f^\star(z_t))},6\sqrt{\log(1/\delta)}
	\right\}\sqrt{\log(1/\delta)}\\
	\leq&
	2\sqrt{{\sum_{t=1}^T\frac{2 Z_t}{{\alpha}}D_\Phi(f_t(z_t),f^\star(z_t))}{\log(1/\delta)}}+6{\log(1/\delta)}\tag{since $\max(a,b)\leq a+b$}\\
	\leq&
	{\sum_{t=1}^T\frac{1}{2}Z_tD_\Phi(f_t(z_t),f^\star(z_t))}+{\frac{4\log(1/\delta)}{\alpha}}+6{\log(1/\delta)}\tag{AM-GM}
\end{align*}
Recall the definition of $\nu_t$, and we conclude that
\begin{align*}
	\sum_{t=1}^TZ_t{D}_{\Phi}\left(f_t(z_t), f^\star(z_t)\right)-\sum_{t=1}^TZ_t\Big(\ell_\phi\big(f_t(z_t), y_t\big)-\ell_\phi\big(f^\star(z_t), y_t\big)\Big)\leq\\\sum_{t=1}^T\frac{1}{2}Z_tD_\Phi(f_t(z_t),f^\star(z_t))+\frac{4\log(1/\delta)}{\alpha}+6{\log(1/\delta)},
\end{align*}
which implies
\begin{align*}
	\frac{1}{2}\sum_{t=1}^TZ_t {D}_{\Phi}\left(f_t(z_t), f^\star(z_t)\right)\leq\sum_{t=1}^TZ_t\Big(\ell_\phi\big(f_t(z_t), y_t\big)-\ell_\phi\big(f^\star(z_t), y_t\big)\Big)+\frac{4\log(1/\delta)}{\alpha}+6{\log(1/\delta)}.
\end{align*}
Plugging this upper bound of Bregman divergence into \eqref{eq:square-bregman}, we obtain that, with probability at least $1-4\delta\log(T)$, for any $\delta<1/e$ and $T>3$, we have
	\begin{align*}
	\sum_{t=1}^TZ_t\big(f^\star(z_t)-f_t(z_t)\big)^2
	\leq\frac{4}{\alpha}\Upsilon+\left(\frac{16}{\alpha^2}+\frac{24}{\alpha}\right)\log(\delta^{-1})=:\beta
	\end{align*}
Finally, we finish the proof by adjusting the coefficient $\delta$ and taking a union bound to obtain the desired result.
\end{proof}

The following lemma is a variant of \citet[Proposition 3]{russo2013eluder}, with the main difference being that (1) the version space is established using the function produced by the oracle instead of the least squares estimator, and (2) the extra multiplicative factor $Z_t$.
\begin{lemma}\label{lem:w-eluder}
	For \Cref{alg:cb}, it holds that
	\begin{align}\label{eq:w-eluder-eq}
		\sum_{t=1}^T Z_t \indic\left\{\sup_{f,f'\in\+F_t}f(x_t,a_t,b_t)-f'(x_t,a_t,b_t)> \epsilon\right\}\leq\left(\frac{4\beta}{\epsilon^2}+1\right)\@{dim}_E(\+F,\epsilon)
	\end{align}
	for any constant $\epsilon>0$,
\end{lemma}
\begin{proof}[Proof of \Cref{lem:w-eluder}]
We first define a subsequence consisting only of the elements for which we made a query in that round. Specifically, we define $((x_{i_1}, a_{i_1}, b_{i_1}), (x_{i_2}, a_{i_2}, b_{i_2}), \dots, (x_{i_k}, a_{i_k}, b_{i_k}))$ where $1\leq i_1 < i_2 < \dots < i_k\leq T$ and $(x_t, a_t, b_t)$ belongs to the subsequence if and only if $Z_t=1$. We further simplify the notation by defining $z_j:=(x_{i_j},a_{i_j},b_{i_j})$ and $f(z_j):=f(x_{i_j},a_{i_j},b_{i_j})$. Then we note that the left-hand side of \eqref{eq:w-eluder-eq} is equivalent to
	\begin{align}\label{eq:equal1}
		\sum_{j=1}^k  \indic\left\{\sup_{f,f'\in\+F_j}f(z_j)-f'(z_j)> \epsilon\right\},
	\end{align}
	and the version space in \Cref{alg:cb} is equal to
	\begin{align}\label{eq:equal2}
		\+F_j=\left\{f\in\+F:\sum_{s=1}^{j-1} \Big(f(z_s)-f_t(z_s)\Big)^2\leq \beta\right\}.
	\end{align}
	
	Hence, it suffice to establish the lower bound for \eqref{eq:equal1} under the version space of \eqref{eq:equal2}. To that end, we make one more simplicication in notation: we denote 
\begin{align*}
	w'_j:=\sup_{f,f'\in\+F_j}f(z_j)-f'(z_j)
\end{align*}

We begin by showing that if $w'_j>\epsilon$ for some $j\in[k]$, then $z_j$ is $\epsilon$-dependent on at most $4\beta/\epsilon^2$ disjoint subsequence of its predecessors. To see this, we note that when $w'_j>\epsilon$, there must exist two function $f,f'\in\+F_j$ such that $f(z_j)-f'(z_j)>\epsilon$. If $z_j$ is $\epsilon$-dependent on a subsequence $(z_{i_1},z_{i_2},\dots,z_{i_n})$ of its predecessors, we must have
\begin{align*}
\sum_{s=1}^{n} \big(f(z_{i_s})-f'(z_{i_s})\big)^2>\epsilon^2.
\end{align*}
Hence, if $z_j$ is $\epsilon$-dependent on $l$ disjoint subsequences, we have 
\begin{align}\label{eq:eluder1}
	\sum_{s=1}^{j-1} \big(f(z_s)-f'(z_s)\big)^2
	>
	l\epsilon^2.
\end{align}
For the left-hand side above, we also have
\begin{align*}
	\sum_{s=1}^{j-1} \big(f(z_s)-f'(z_s)\big)^2
	\leq
	2\sum_{s=1}^{j-1} \big(f(z_s)-f_t(z_s)\big)^2
	+
	2\sum_{s=1}^{j-1} \big(f_t(z_s)-f'(z_s)\big)^2
	\leq4\beta\numberthis\label{eq:eluder2}
\end{align*}
where the first inequality holds since $(a+b)^2\leq2(a^2+b^2)$ for any $a,b$, and the second inequality holds by \eqref{eq:equal2}. Combining \eqref{eq:eluder1} and \eqref{eq:eluder2}, we get that $l\leq4\beta/\epsilon^2$.
	
	Next, we show that for any sequence $(z'_1,\dots,z'_\tau)$, there is at least one element that is $\epsilon$-dependent on at least $\tau/d-1$ disjoint subsequence of its predecessors, where $d:=\@{dim}_E(\+F,\epsilon)$. To show this, let $m$ be the integer satisfying $md+1\leq \tau\leq md+d$. We will construct $m$ disjoint subsequences, $B_1,\dots,B_m$. At the beginning, let $B_i=(z'_i)$ for $i\in[m]$. If $z'_{m+1}$ is $\epsilon$-dependent on each subsequence $B_1,\dots,B_m$, then we are done. Otherwise, we select a subsequence $B_i$ which $z'_{m+1}$ is $\epsilon$-independent of and append $z'_{m+1}$ to $B_i$. We repeat this process for all elements with indices $j>m+1$ until either $z'_j$ is $\epsilon$-dependent on each subsequence or $j=\tau$. For the latter, we have $\sum_{i=1}^m |B_i|\geq md$, and since each element of a subsequence $B_i$ is $\epsilon$-independent of its predecesors, we must have $|B_i|=d$ for all $i$. Then, $z_\tau$ must be $\epsilon$-dependent on each subsequence by the definition of eluder dimension.
	
	Finally, let's take the sequence $(z'_1,\dots z'_\tau)$ to be the subsequence of $(z_1,\dots,z_k)$ consisting of elements $z_j$ for which $w_j'>\epsilon$. As we have established, we have (1) each $z'_j$ is $\epsilon$-dependent on at most $4\beta/\epsilon^2$ disjoint subsequences, and (2) some $z'_j$ is $\epsilon$-dependent on at least $\tau/d-1$ disjoint subsequences. Therefore, we must have $\tau/d-1\leq 4\beta/\epsilon^2$, implying that $\tau\leq(4\beta/\epsilon^2+1)d$.
\end{proof}

The following lemma is adopted from \citet[Lemma 3]{saha2022efficient}.
\begin{lemma}\label{lem:igw}
	For any function $f\in\+F$ and any context $x\in\+X$, the following convex program of $p\in\Delta(\+A)$ is always feasible:
	\begin{align*}
		\forall a\in\+A:
		\sum_{b}f(x,a,b)p(b)+\frac{2}{\gamma p(a)}\leq\frac{5A}{\gamma}.
	\end{align*}
	Furthermore, any solution $p$ satisfies:
	\begin{align*}
		\E_{a\sim p}\Big[f^\star(x,\pi_{f^\star}(x),a)\Big]
		\leq 
		\frac{\gamma}{4}\E_{a,b\sim p}\Big[\big(f(x,a,b)-f^\star(x,a,b)\big)^2\Big]+\frac{5A}{\gamma}
	\end{align*}
	whenever $\gamma\geq 2A$.
\end{lemma}

\begin{lemma}\label{lem:igw-r-version}
	Assume that for each $f\in\+F$, there exists an associated function $r:\+X\times\+A\rightarrow[0,1]$ such that $f(x,a,b)=r(x,a)-r(x,b)$ for any $x\in\+X$ and $a,b\in\+A$. In this case, for any context $x\in\+X$, if we define $p$ as
	\begin{align*}
		p(a)=
		\begin{cases}
			\frac{1}{A+\gamma\big(r(x,\pi_f(x))-r(x,a)\big)} & a\neq\pi_f(x)\\
			1-\sum_{a\neq\pi_f(x)}p(a) & a=\pi_f(x)
		\end{cases},
	\end{align*}
	then we have
	\begin{align*}
		\E_{a\sim p}\Big[f^\star(x,\pi_{f^\star}(x),a)\Big]
		\leq
		\gamma\E_{a,b\sim p}\Big[\Big(f(x,a,b)-f^\star(x,a,b)\Big)^2\Big]+\frac{A}{\gamma}
	\end{align*}
\end{lemma}
\begin{proof}[Proof of \cref{lem:igw-r-version}]
	Fix any $b\in\+A$. Then, the distribution $p$ can be rewritten as
	\begin{align*}
		p(a)=
		\begin{cases}
			\left(
			A+2\gamma
			\left(
			\frac{r(x,\pi_f(x))-r(x,b)+1}{2}
			-
			\frac{r(x,a)-r(x,b)+1}{2}
			\right)
			\right)^{-1} 
			& a\neq\pi_f(x)\\
			1-\sum_{a\neq\pi_f(x)}p(a) 
			& a=\pi_f(x)
		\end{cases}.
	\end{align*}
	Therefore, denoting $f^\star(x,a,b)=r^\star(x,a)-r^\star(x,b)$ for some function $r^\star$, we have
	\begin{align*}
		\E_{a\sim p}\Big[f^\star(x,\pi_{f^\star}(x),a)\Big]
		=&\E_{a\sim p}\Big[r^\star(x,\pi_{f^\star}(x))-r^\star(x,a)\Big]\\
		=&2\E_{a\sim p}\left[\frac{r^\star(x,\pi_{f^\star}(x))-r^\star(x,b)+1}{2}-\frac{r^\star(x,a)-r^\star(x,b)+1}{2}\right]\\
		\leq&
		2\cdot2\gamma\E_{a\sim p}\left[\left(\frac{r(x,a)-r(x,b)+1}{2}-\frac{r^\star(x,a)-r^\star(x,b)+1}{2}\right)^2\right]+\frac{A}{\gamma}\\
		=&
		\gamma\E_{a\sim p}\Big[\Big(f(x,a,b)-f^\star(x,a,b)\Big)^2\Big]+\frac{A}{\gamma}
	\end{align*}
	where for the inequality above we invoked \Cref{lem:igw-general} with $\^y(a)=(r(x,a)-r(x,b)+1)/2$ and $y^\star(a)=(r^\star(x,a)-r^\star(x,b)+1)/2$. We note that the above holds for any $b\in\+A$. Hence, we complete the proof by sampling $b\sim p$.
\end{proof}

\begin{lemma}\label{lem:sum-zw}
	Assume $f^\star\in\+F_t$ for all $t\in[T]$. Suppose there exists some $t'\in[T]$ such that $\lambda_t=0$ for all $t\leq t'$. Then we have 
	\begin{align*}
	\sum_{t=1}^{t'} Z_tw_t\leq
	56 A^2\beta\cdot\frac{\@{dim}_E\left(\+F,\Delta\right)}{\Delta}\cdot\log(2/(\delta\Delta))
	\end{align*}
\end{lemma}
with probability at least $1-\delta$.
\begin{proof}
Since $f^\star\in\+F_t$, we always have $\pi_{f^\star}(x_t)\in\+A_t$ for all $t\in[T]$. Hence, whenever $Z_t$ is zero, we have $\+A_t=\{\pi_{f^\star}(x_t)\}$ and thus we do not incur any regret. Hence, we know $Z_tw_t$ is either 0 or at least $\Delta$ by \Cref{lem:width-lower-bound}. Let us fix an integer $m>1/\Delta$, whose value will be specified later. We divide the interval $[\Delta,1]$ into bins of width $1/m$ and conduct a refined study of the sum of $Z_t w_t$:
\begin{align*}
	\sum_{t=1}^{t'} Z_tw_t
	\leq&\sum_{t=1}^{t'} \sum_{j=0}^{\left(1-\Delta\right)m-1}Z_tw_t\cdot\indic\left\{Z_tw_t\in\left[\Delta+\frac{j}{m},\,\Delta+\frac{j+1}{m}\right]\right\}\\
	\leq&\sum_{j=0}^{\left(1-\Delta\right)m-1}\left(\Delta+\frac{j+1}{m}\right)\sum_{t=1}^{t'} Z_t\indic\left\{w_t\geq\Delta+\frac{j}{m}\right\}\\
	=&\sum_{j=0}^{\left(1-\Delta\right)m-1}\left(\Delta+\frac{j+1}{m}\right)\sum_{t=1}^{t'} Z_t\indic\left\{\sup_{a,b\in\+A_t}\sup_{f,f'\in\+F_t}f(x_t,a,b)-f'(x_t,a,b)\geq\Delta+\frac{j}{m}\right\}\\
	=&\sum_{j=0}^{\left(1-\Delta\right)m-1}\left(\Delta+\frac{j+1}{m}\right)\sum_{t=1}^{t'} Z_t\sup_{a,b\in\+A_t}\indic\left\{\sup_{f,f'\in\+F_t}f(x_t,a,b)-f'(x_t,a,b)\geq\Delta+\frac{j}{m}\right\}\\
	\leq&\sum_{j=0}^{\left(1-\Delta\right)m-1}\left(\Delta+\frac{j+1}{m}\right)\sum_{t=1}^{t'} Z_t\sum_{a,b}\indic\left\{\sup_{f,f'\in\+F_t}f(x_t,a,b)-f'(x_t,a,b)\geq\left(\Delta+\frac{j}{m}\right)\right\}\\
	\leq&\sum_{j=0}^{\left(1-\Delta\right)m-1}\left(\Delta+\frac{j+1}{m}\right)A^2 \underbrace{\sum_{t=1}^{t'} Z_t \E_{a,b\sim p_t}\indic\left\{\sup_{f,f'\in\+F_t}f(x_t,a,b)-f'(x_t,a,b)\geq\left(\Delta+\frac{j}{m}\right)\right\}}_{(*)}
\end{align*}
where in the third inequality we replace the supremum over $a,b$ by the summation over $a,b$, and in the last inequality we further replace it by the expectation. Here recall that $p_t(a)$ is uniform when $\lambda_t=0$, leading to the extra $A^2$ factor. To deal with $(*)$, we first apply \Cref{lem:freedman} to recover the empirical $a_t$ and $b_t$, and then apply \Cref{lem:w-eluder} to get an upper bound via the eluder dimension:
\begin{align*}
	(*)\leq&2\sum_{t=1}^{t'} Z_t \indic\left\{\sup_{f,f'\in\+F_t}f(x_t,a_t,b_t)-f'(x_t,a_t,b_t)\geq\left(\Delta+\frac{j}{m}\right)\right\}+8\log(\delta^{-1})\\
	\leq&2\left(\frac{4\beta}{\left(\Delta+\frac{j}{m}\right)^2}+1\right)\@{dim}_E\left(\+F;\Delta\right)+8\log(\delta^{-1})\\
	\leq&\frac{10\beta}{\left(\Delta+\frac{j}{m}\right)^2}\cdot\@{dim}_E\left(\+F;\Delta\right)+8\log(\delta^{-1})
\end{align*}
with probability at least $1-\delta$. Plugging $(*)$ back, we obtain
\begin{align*}
	\sum_{t=1}^{t'} Z_tw_t
	\leq&\sum_{j=0}^{\left(1-\Delta\right)m-1}\left(\Delta+\frac{j+1}{m}\right)\cdot\frac{10A^2\beta}{\left(\Delta+\frac{j}{m}\right)^2}\cdot\@{dim}_E\left(\+F;\Delta\right)+8mA^2\log(\delta^{-1})\\
	=&10A^2\beta\cdot \@{dim}_E\left(\+F,\Delta\right)\sum_{j=0}^{\left(1-\Delta\right)m-1}\frac{\Delta+\frac{j+1}{m}}{\left(\Delta+\frac{j}{m}\right)^2}+8mA^2\log(\delta^{-1})\\
	\leq&10A^2\beta\cdot \@{dim}_E\left(\+F,\Delta\right)\left(\frac{\Delta+1/m}{\Delta^2}+\sum_{j=1}^{\left(1-\Delta\right)m-1}\frac{2}{\Delta+\frac{j}{m}}\right)+8mA^2\log(\delta^{-1})\\
	\leq&10A^2\beta\cdot \@{dim}_E\left(\+F,\Delta\right)\sum_{j=0}^{\left(1-\Delta\right)m-1}\frac{2}{\Delta+\frac{j}{m}}+8mA^2\log(\delta^{-1})\\
	\leq&20A^2\beta\cdot \@{dim}_E\left(\+F,\Delta\right)\sum_{j=0}^{\left(1-\Delta\right)m-1}\int_{j-1}^j\frac{1}{\Delta+\frac{x}{m}} \d x+8mA^2\log(\delta^{-1})\\
	=&20A^2\beta\cdot \@{dim}_E\left(\+F,\Delta\right)\int_{-1}^{(1-\Delta)m-1}\frac{1}{\Delta+\frac{x}{m}} \d x+8mA^2\log(\delta^{-1})\\
	=&20A^2\beta\cdot \@{dim}_E\left(\+F,\Delta\right)\cdot m\log\left(\frac{1}{\Delta-m^{-1}}\right)+8mA^2\log(\delta^{-1})
\end{align*}
where for the second inequality, we use the fact that $(j+1)/m\leq 2j/m$ for any $j\geq 1$; for the third inequality, we assume $m>1/\Delta$. Setting $m=2/\Delta$, we arrive at
\begin{align*}
	\sum_{t=1}^{t'} Z_tw_t
	\leq&
	40A^2\beta\cdot\frac{\@{dim}_E\left(\+F,\Delta\right)}{\Delta}\cdot\log(2/\Delta)+16A^2\log(\delta^{-1})/\Delta\\
	\leq&
	56 A^2\beta\cdot\frac{\@{dim}_E\left(\+F,\Delta\right)}{\Delta}\cdot\log(2/(\delta\Delta)),
\end{align*}
which completes the proof.
\end{proof}

\begin{lemma}\label{lem:lambda-all-0}
Whenever 
$$
56A^2\beta\cdot\@{dim}_E\left(\+F,\Delta\right)\cdot\log(2/(\delta\Delta))/\Delta<\sqrt{AT/\beta},
$$
we have $\lambda_1=\lambda_2=\dots=\lambda_T=0$ with probability at least $1-\delta$.
\end{lemma}
\begin{proof}[Proof of \Cref{lem:lambda-all-0}]
	We prove it via contradiction. Assume the inequality holds but there exists $t'$ for which $\lambda_{t'}=1$. Without loss of generality, we assume that $\lambda_t=0$ for all $t<t'$, namely that $t'$ is the first time that $\lambda_t$ is 1. Then by definition of $\lambda_{t'}$, we have
	\begin{align*}
		\sum_{s=1}^{t'-1} Z_sw_s\geq\sqrt{AT/\beta}.
	\end{align*}
	On the other hand, by \Cref{lem:sum-zw}, we have
	\begin{align*}
		\sum_{s=1}^{t'-1} Z_s w_s\leq
	56A^2\beta\cdot\frac{\@{dim}_E\left(\+F,\Delta\right)}{\Delta}\cdot\log(2/(\delta\Delta)).
	\end{align*}
	The combination of the above two inequalities contradicts with the conditions.
\end{proof}

\subsection{Proof of Lemma~\ref{lem:optimal-action-exist}}
\begin{proof}[Proof of \Cref{lem:optimal-action-exist}]
	We prove it via contradiction. If no such arm exists, meaning that for any arm $a$, there exists an arm $b$ such that $f^\star(x,a,b)<0$. Then we can find a sequence of arms $(a_1,a_2,\dots,a_k)$ such that $f^\star(x,a_i,a_{i+1})<0$ for any $i=1,\dots,k-1$ and $f^\star(x,a_k,a_1)<0$, which contradicts with the transitivity (\Cref{asm:properties}).
\end{proof}

\subsection{Proof of Theorem~\ref{thm:cb-regret}}\label{sec:pf-thm-cb-regret}

We begin by showing the worst-case regret upper bound.

\begin{lemma}[Worst-case regret upper bound]\label{lem:worst-case-regret-ub}
	For \Cref{alg:cb}, assume $f^\star\in\+F_t$ for all $t\in[T]$. Then, we have
	\begin{align*}
		\@{Regret}^{\@{CB}}_T\leq68\sqrt{AT\beta}\cdot\log(4\delta^{-1})
	\end{align*}
	with probability at least $1-\delta$.
\end{lemma}
\begin{proof}[Proof of \Cref{lem:worst-case-regret-ub}]

We recall that the regret is defined as
\begin{align*}
		\@{Regret}^{\@{CB}}_T=\sum_{t=1}^{T}\big(
		f^\star(x_t,\pi_{f^\star}(x_t),a_t)+
				f^\star(x_t,\pi_{f^\star}(x_t),b_t)\big).
\end{align*}
Since $a_t$ and $b_t$ are always drawn independently from the same distribution in \Cref{alg:cb}, we only need to consider the regret of the $a_t$ part in the following proof for brevity --- multiplying the result by two would yield the overall regret.

	We first observe the definition of $\lambda_t$ in \Cref{alg:cb}: the left term $\sum_{s=1}^{t-1}Z_s w_s$ in the indicator is non-decreasing in $t$ while the right term remains constant. This means that there exists a particular time step $t'\in[T]$ dividing the time horizon into two phases: $\lambda_t=0$ for all $t\leq t'$ and $\lambda_t=1$ for all $t>t'$. Now, we proceed to examine these two phases individually.
	
	For all rounds before or on $t'$, we can compute the expected partial regret as
	\begin{align*}
		\sum_{t=1}^{t'}\E_{a\sim p_t}\big[f^\star(x_t,\pi_{f^\star}(x_t),a)\big]
		=
		\sum_{t=1}^{t'}Z_t\E_{a\sim p_t}\big[f^\star(x_t,\pi_{f^\star}(x_t),a)\big]
		\leq&\sum_{t=1}^{t'}Z_tw_t\leq\sqrt{AT\beta},\numberthis\label{eq:worst-regret-1}
	\end{align*}
	where the equality holds since we have $\+A_t=\{\pi_{f^\star}(x_t)\}$ whenever $Z_t=0$ under the condition that $f^\star\in\+F_t$, and thus we don't incur regret in this case. The first inequality is \Cref{lem:regret-bounded-by-w}, and the second inequality holds by the definition of $\lambda_t$ and the condition that $\lambda_{t}=0$.
	
	On the other hand, for all rounds after $t'$, we have 
	\begin{align*}
		&\sum_{t=t'+1}^{T}\E_{a\sim p_t}\big[f^\star(x_t,\pi_{f^\star}(x_t),a)\big]\\
		=&\sum_{t=t'+1}^{T}Z_t\E_{a\sim p_t}\big[f^\star(x_t,\pi_{f^\star}(x_t),a)\big]\\
		\leq&\sum_{t=t'+1}^{T}Z_t\left(\frac{5A}{\gamma_t}+\frac{\gamma_t}{4}\E_{a,b\sim p_t}\Big[\big(f^\star(x_t,a,b)-f_t(x_t,a,b)\big)^2\Big]\right)\\
		=&\sum_{t=t'+1}^{T}Z_t\left(\frac{5A}{\sqrt{AT/\beta}}+\frac{\sqrt{AT/\beta}}{4}\E_{a,b\sim p_t}\Big[\big(f^\star(x_t,a,b)-f_t(x_t,a,b)\big)^2\Big]\right)\\
		\leq&5\sqrt{AT\beta}+\frac{\sqrt{AT/\beta}}{4}\sum_{t=t'+1}^{T} Z_t \E_{a,b\sim p_t}\Big[\big(f^\star(x_t,a,b)-f_t(x_t,a,b)\big)^2\Big]\\
		\leq&5\sqrt{AT\beta}+\frac{\sqrt{AT/\beta}}{2}\sum_{t=t'+1}^{T} Z_t \big(f^\star(x_t,a_t,b_t)-f_t(x_t,a_t,b_t)\big)^2+8\sqrt{AT/\beta}\cdot\log(4\delta^{-1})\\
		\leq&5\sqrt{AT\beta}+\frac{\sqrt{AT\beta}}{2}+8\sqrt{AT/\beta}\cdot\log(4\delta^{-1}).\numberthis\label{eq:worst-regret-2}
	\end{align*}
	where the first inequality holds by \Cref{lem:igw} (or \Cref{lem:igw-r-version} for specific function classes), the second equality is by the definition of $\gamma_t$, the third inequality is by \Cref{lem:freedman}, and the fourth inequality holds by \Cref{lem:pointwise-bound}.
	
	Putting the two parts, \eqref{eq:worst-regret-1} and \eqref{eq:worst-regret-2}, together, we arrive at
	\begin{align*}
		\sum_{t=1}^{T}\E_{a\sim p_t}\big[f^\star(x_t,\pi_{f^\star}(x_t),a)\big]
		\leq7\sqrt{AT\beta}+8\sqrt{AT/\beta}\cdot\log(4\delta^{-1})
		\leq15\sqrt{AT\beta}\cdot\log(4\delta^{-1}).
	\end{align*}
	Now we apply \Cref{lem:freedman} again. The following holds with probability at least $1-\delta/2$,
	\begin{align*}
		\sum_{t=1}^{T}f^\star(x_t,\pi_{f^\star}(x_t),a_t)
		\leq
		2\sum_{t=1}^{T}\E_{a\sim p_t}\big[f^\star(x_t,\pi_{f^\star}(x_t),a)\big]
		+4\log(4\delta^{-1})
		\leq
		34\sqrt{AT\beta}\cdot\log(4\delta^{-1}).
	\end{align*}
	The above concludes the regret of the $a_t$ part. The regret of the $b_t$ can be shown in the same way. Adding them together, we conclude that
	\begin{align*}
		\@{Regret}^{\@{CB}}_T=\sum_{t=1}^{T}\big(
		f^\star(x_t,\pi_{f^\star}(x_t),a_t)+
				f^\star(x_t,\pi_{f^\star}(x_t),b_t)\big)
		\leq68\sqrt{AT\beta}\cdot\log(4\delta^{-1}).
	\end{align*}
\end{proof}

\begin{lemma}[Instance-dependent regret upper bound]\label{lem:ins-depend-regret-ub}
	For \Cref{alg:cb}, assume $f^\star\in\+F_t$ for all $t\in[T]$. Then, we have
	\begin{align*}
		\@{Regret}^{\@{CB}}_T\leq
		3808 A^2\beta^2\cdot\frac{\@{dim}_E\left(\+F,\Delta\right)}{\Delta}\cdot\log^2(4/(\delta\Delta)) 
	\end{align*}
	with probability at least $1-\delta$.
\end{lemma}
\begin{proof}[Proof of \Cref{lem:ins-depend-regret-ub}]
	We consider two cases. First, when 
	\begin{equation}\label{eq:regret-cases1}
	56 A^2\beta\cdot\frac{\@{dim}_E\left(\+F,\Delta\right)}{\Delta}\cdot\log(2/(\delta\Delta))<\sqrt{AT/\beta},
	\end{equation}
	we invoke \Cref{lem:lambda-all-0} and get that $\lambda_t=0$ for all $t\in[T]$. Hence, we have
	\begin{align*}
		\@{Regret}^{\@{CB}}_T
		=&\sum_{t=1}^{T}\big(
		f^\star(x_t,\pi_{f^\star}(x_t),a_t)+
				f^\star(x_t,\pi_{f^\star}(x_t),b_t)\big)\\
		\leq&2\sum_{t=1}^T Z_t w_t\\
		\leq&112A^2\beta\cdot\frac{\@{dim}_E\left(\+F,\Delta\right)}{\Delta}\cdot\log(2/(\delta\Delta))\\
		\leq&3808 A^2\beta^2\cdot\frac{\@{dim}_E\left(\+F,\Delta\right)}{\Delta}\cdot\log^2(4/(\delta\Delta))
	\end{align*}
	where the first inequality is by \Cref{lem:regret-bounded-by-w} and the fact that we incur no regret when $Z_t=0$ since $f^\star\in\+F_t$. The second inequality is by \Cref{lem:sum-zw}.
	
	On the other hand, when the contrary of \eqref{eq:regret-cases1} holds, i.e., 
	\begin{equation}\label{eq:regret-cases2}
	56 A^2\beta\cdot\frac{\@{dim}_E\left(\+F,\Delta\right)}{\Delta}\cdot\log(2/(\delta\Delta))\geq\sqrt{AT/\beta},
	\end{equation}
	applying \Cref{lem:worst-case-regret-ub}, we have
	\begin{align*}
		\@{Regret}^{\@{CB}}_T
		\leq&68\sqrt{AT\beta}\cdot\log(4\delta^{-1})\\
		=&68\beta\cdot\log(4\delta^{-1})\cdot\sqrt{AT/\beta}\\
		\leq&68\beta\cdot\log(4\delta^{-1})\cdot
	56A^2\beta\cdot\frac{\@{dim}_E\left(\+F,\Delta\right)}{\Delta}\cdot\log(2/(\delta\Delta))\\
		\leq& 3808 A^2\beta^2\cdot\frac{\@{dim}_E\left(\+F,\Delta\right)}{\Delta}\cdot\log^2(4/(\delta\Delta))
	\end{align*}
	where we apply the condition \eqref{eq:regret-cases2} in the second inequality.
\end{proof}

\begin{lemma}[Query complexity]\label{lem:query-bound}
	For \Cref{alg:cb}, assume $f^\star\in\+F_t$ for all $t\in[T]$. Then, we have
		\begin{align*}
			\@{Queries}^{\@{CB}}_T\leq \min\left\{T,\, 3136 A^3 \beta^3 \frac{\@{dim}^2_E\left(\+F,\Delta\right)}{\Delta^2}\cdot\log^2(2/(\delta\Delta))\right\}
		\end{align*}
		with probability at least $1-\delta$.
	\end{lemma}
	\begin{proof}[Proof of \Cref{lem:query-bound}]
	
	We consider two cases. First, when
	\begin{align}\label{eq:query-case}
	56 A^2\beta\cdot\frac{\@{dim}_E\left(\+F,\Delta\right)}{\Delta}\cdot\log(2/(\delta\Delta))<\sqrt{AT/\beta}	
	\end{align}
	we can invoke \Cref{lem:lambda-all-0} and get that $\lambda_t=0$ for all $t\in[T]$. Hence,
	\begin{align*}
		\@{Queries}^{\@{CB}}_T
		=&\sum_{t=1}^T Z_t\\
		=&\sum_{t=1}^T Z_t\indic\{w_t\geq \Delta\}\\
		=&\sum_{t=1}^TZ_t\sup_{a,b\in\+A_t}\indic\left\{\sup_{f,f'\in\+F_t}f(x_t,a,b)-f'(x_t,a,b)\geq\Delta\right\}\\
		\leq&\sum_{t=1}^TZ_t\sum_{a,b}\indic\left\{\sup_{f,f'\in\+F_t}f(x_t,a,b)-f'(x_t,a,b)\geq\Delta\right\}\\
		\leq&A^2\underbrace{\sum_{t=1}^TZ_t \E_{a,b\sim p_t}\indic\left\{\sup_{f,f'\in\+F_t}f(x_t,a,b)-f'(x_t,a,b)\geq\Delta\right\}}_{(*)}
	\end{align*}
	where the second equality is by \Cref{lem:width-lower-bound}, the second inequality holds as $p_t(a)$ is uniform for any $a,b$ when $\lambda_t=0$. We apply \Cref{lem:freedman} and \Cref{lem:w-eluder} to $(*)$ and obtain
	\begin{align*}
		(*)\leq& 2\sum_{t=1}^TZ_t \indic\left\{\sup_{f,f'\in\+F_t}f(x_t,a_t,b_t)-f'(x_t,a_t,b_t)\geq\Delta\right\}+8\log(\delta^{-1})\\
		\leq& 2\left(\frac{4\beta}{\Delta^2}+1\right)\@{dim}_E(\+F;\Delta)+8\log(\delta^{-1})\\
		\leq& \frac{10\beta}{\Delta^2}\cdot\@{dim}_E(\+F;\Delta)+8\log(\delta^{-1}).
	\end{align*}
	Plugging this back, we obtain
	\begin{align*}
		\@{Queries}^{\@{CB}}_T
		\leq&\frac{10A^2\beta}{\Delta^2}\cdot\@{dim}_E(\+F;\Delta)+8A^2\log(\delta^{-1})\\
		\leq&3136 A^3 \beta^3 \frac{\@{dim}^2_E\left(\+F,\Delta\right)}{\Delta^2}\cdot\log^2(2/(\delta\Delta)).
	\end{align*}
	
	On the other hand, when the contrary of \eqref{eq:query-case} holds, i.e., 
	\begin{align*}
	56 A^2\beta\cdot\frac{\@{dim}_E\left(\+F,\Delta\right)}{\Delta}\cdot\log(2/(\delta\Delta))\geq\sqrt{AT/\beta}.
	\end{align*}
	Squaring both sides, we obtain
	\begin{align*}
	3136 A^4 \beta^2 \frac{\@{dim}^2_E\left(\+F,\Delta\right)}{\Delta^2}\cdot\log^2(2/(\delta\Delta))
	\geq AT/\beta
	\end{align*}
	which leads to
	\begin{equation*}
	T
	\leq 3136 A^3 \beta^3 \frac{\@{dim}^2_E\left(\+F,\Delta\right)}{\Delta^2}\cdot\log^2(2/(\delta\Delta)).
	\end{equation*}
	We note that we always have $\@{Queries}^{\@{CB}}_T\leq T$, and thus,
	\begin{equation*}
	\@{Queries}^{\@{CB}}_T\leq
	T
	\leq 3136 A^3 \beta^3 \frac{\@{dim}^2_E\left(\+F,\Delta\right)}{\Delta^2}\cdot\log^2(2/(\delta\Delta)).
	\end{equation*}
	Hence, we complete the proof.
	\end{proof}
	
	Having established the aforementioned lemmas, we are now able to advance towards the proof of Theorem 1.
	
	\begin{proof}[Proof of Theorem~\ref{thm:cb-regret}]
	By \Cref{lem:pointwise-bound} and the construction of version spaces $\+F_t$ in \Cref{alg:cb}, we have $f^\star\in\+F_t$ for all $t\in[T]$ with probability at least $1-\delta$. Then, the rest of the proof follows from \Cref{lem:worst-case-regret-ub,lem:ins-depend-regret-ub,lem:query-bound}.
	\end{proof}

\subsection{Proof of Theorem~\ref{thm:lower-bound}}

In this section, we will prove the following theorem, which is stronger than \Cref{thm:lower-bound}.
\begin{theorem}[Lower bounds]\label{thm:lower-bound-stronger}
The following two claims hold:
\begin{enumerate}
	\item[(1)] for any algorithm, there exists an instance that leads to $\@{Regret}^{\@{CB}}_T=\Omega(\sqrt{AT})$;
	\item[(2)] for any algorithm achieving a worse-case expected regret upper bound in the form of $\E[\@{Regret}^{\@{CB}}_T]= O(\sqrt{A}\cdot T^{1-\beta})$ for some $\beta>0$, there exists an instance with gap $\Delta=\sqrt{A}\cdot T^{-\beta}$ that results in $\E[\@{Regret}^{\@{CB}}_T]=\Omega(A/\Delta)=\Omega(\sqrt{A}\cdot T^{\beta})$ and $\E[\@{Queries}^{\@{CB}}_T]=\Omega(A/\Delta^2)=\Omega(T^{2\beta})$.
\end{enumerate}
\end{theorem}
We observe that \Cref{thm:lower-bound} can be considered as a corollary of the above theorem when setting $\beta=1/2$.

In what follows, we will first demonstrate lower bounds in the setting of \textit{multi-armed bandits (MAB) with active queries} and subsequently establish a reduction from it to contextual dueling bandits in order to achieve these lower bounds.  We start by formally defining the setting of MAB with active queries below. 

\textbf{Multi-armed bandits with active queries.}
We consider a scenario where there exist $A$ arms. Each arm $a$ is assumed to yield a binary reward (0 or 1), which is sampled from a Bernoulli distribution $\text{Bern}(\-r_a)$, where $\-r_a$ denotes the mean reward associated with arm $a$.The arm with the highest mean reward is denoted by $a^\star\coloneqq\argmax_a \-r_a$. Let $\Delta_a\coloneqq \-r_{a^\star}-\-r_{a}$ denote the gap of arm $a\in[A]$. The interaction proceeds as follows: at each round $t\in[T]$, we need to pull an arm but can choose whether to receive the reward signal (denote this choice by $Z_t$). The objective is to minimize two quantities: the regret and the number of queries,
\begin{equation}\label{eq:mab-regret-query}
	\@{Regret}_T=\sum_{t=1}^T \Delta_{a_t},\quad\@{Queries}_T=\sum_{t=1}^T Z_t.
\end{equation}

Towards the lower bounds, we will start with a bound on the KL divergence over distributions of runs under two different bandits. This result is a variant of standard results which can be found in many bandit literature (e.g., \citet{lattimore2020bandit}).

\begin{lemma}\label{lem:kl-decompose}
Let $I_1$ and $I_2$ be two instances of MAB. We define $p_1$ and $p_2$ as their respective distributions over the outcomes of all pulled arms and reward signals when a query is made. Concretely, $p_1$ and $p_2$ are measuring the probability of outcomes (denoted by $O$) in the following form:
	\begin{align*}
		O=\big(Z_1,a_1,(r_1),\dots,Z_T,a_T,(r_T)\big)
	\end{align*}
	where the reward $r_t$ is included only when $Z_t=1$, and we added parentheses above to indicate this point. We denote $\Pr_1$ (resp. $\Pr_2$) as the reward distribution of $I_1$ (resp. $I_2$).  We define $\-n_a=\sum_{t=1}^T Z_t \indic\{a_t=a\}$ as the number of times arm $a$ is pulled when making a query. Then, given any algorithm $\#A$, the Kullback–Leibler divergence between $p_1$ and $p_2$ can be decomposed in the following way
	\begin{align*}
		\@{KL}(p_1,p_2)&=\sum_{a=1}^A \E_{p_1}[\-n_{a}] \cdot \@{KL}\big(\Pr_1(r\given a),\Pr_2(r\given a)\big).
	\end{align*}
\end{lemma}
\begin{proof}[Proof of \Cref{lem:kl-decompose}]
	We define the conditional distribution
	\begin{align*}
		\overline\Pr_1(r_t\given Z_t,a_t)
		\begin{cases}
			\Pr_1(r_t\given a_t) & \text{if }Z_t=1\\
			1 & \text{if }Z_t=0
		\end{cases},
	\end{align*}
	and similarly for $\overline\Pr_2$. Additionally, we denote $\Pr_{\#A}$ as the probability associated with algorithm $\#A$. Then, for any outcome $O$, we have
	\begin{align*}
		p_1(O)=\prod_{t=1}^T \Pr_{\#A}\big(Z_t,a_t\given Z_1,a_1,(r_1),\dots,Z_{t-1},a_{t-1},(r_{t-1})\big) \overline\Pr_1(r_t\given Z_t,a_t),
	\end{align*}
	and we can write $p_2(O)$ in a similar manner. Hence,
	\begin{align*}
		\@{KL}(p_1,p_2)&=\E_{O\sim p_1}\left[
		\log\left(
		\frac{\prod_{t=1}^T \Pr_{\#A}\big(Z_t,a_t\given Z_1,a_1,(r_1),\dots,Z_{t-1},a_{t-1},(r_{t-1})\big) \overline\Pr_1(r_t\given Z_t,a_t)}{\prod_{t=1}^T \Pr_{\#A}\big(Z_t,a_t\given Z_1,a_1,(r_1),\dots,Z_{t-1},a_{t-1},(r_{t-1})\big) \overline\Pr_2(r_t\given Z_t,a_t)}
		\right)
		\right]\\
		&=\E_{O\sim p_1}\left[
		\sum_{t=1}^T\log\left(
		\frac{\overline\Pr_1(r_t\given Z_t,a_t)}{\overline\Pr_2(r_t\given Z_t,a_t)}
		\right)
		\right]\\
		&=\E_{O\sim p_1}\left[
		\sum_{t=1}^T Z_t\log\left(
		\frac{\Pr_1(r_t\given a_t)}{\Pr_2(r_t\given a_t)}
		\right)
		\right]\\
		&=\E_{O\sim p_1}\left[
		\sum_{t=1}^T Z_t\E_{r_t\sim \Pr_1(\cdot\given a_t)}\left[\log\left(
		\frac{\Pr_1(r_t\given a_t)}{\Pr_2(r_t\given a_t)}
		\right)
		\right]
		\right]\\
		&=\E_{O\sim p_1}\left[
		\sum_{t=1}^T Z_t \cdot \@{KL}\big(\Pr_1(\cdot\given a_t),\Pr_2(\cdot\given a_t)\big)
		\right]\\
		&=\sum_{a=1}^A \E_{O\sim p_1}[\-n_{a}] \cdot \@{KL}\big(\Pr_1(\cdot\given a_t),\Pr_2(\cdot\given a_t)\big)
	\end{align*}
	where the third equality holds by the definition of $\overline\Pr_1$ and $\overline\Pr_2$.
\end{proof}

The following lemma establishes lower bounds for MAB with active queries. It presents a trade-off between the regret and the number of queries.

\begin{lemma}\label{lem:mab-tradeoff}
	Let $\+I$ denote the set of all MAB instances. Assume \Alg{} is an algorithm that achieves the following worst-case regret upper bound for some $C$ and $\beta$:
	\begin{equation*}
		\E\big[\@{Regret}_T\big]\leq C T^{1-\beta},
	\end{equation*}
	for all $I\in\+I$. 
	Then, for any MAB instance $I\in\+I$, the regret and the number of queries made by algorithm \Alg{} are lower bounded:
	\begin{equation*}
		\E\big[\@{Regret}_T\big]\geq 
		\sum_{a\neq a^\star}\frac{\zeta}{\Delta_{a}}\log\left(\frac{\Delta_{a}}{4 C T^{-\beta}}\right)
		,\quad
		\E\big[\@{Queries}_T\big]\geq 
		\sum_{a\neq a^\star}\frac{\zeta}{\Delta^2_{a}}\log\left(\frac{\Delta_{a}}{4 C T^{-\beta}}\right)
	\end{equation*}
	where the coefficient $\zeta=\min_a\min\{\-r_a,1-\-r_a\}$ depends on the instance $I$.
\end{lemma}
\begin{proof}[Proof of \Cref{lem:mab-tradeoff}]
For any MAB instance $I$ and any arm $a^\dagger$, we define a corresponding MAB instance $I'$ as follows. Denote $\-r$ and $\-r'$ as the mean reward of $I$ and $I'$, respectively. For $I'$, we set the mean reward $\-r'_a=\-r_a$ for any $a\neq a^\dagger$ and $\-r'_{a^\dagger}=\-r_{a^\dagger}+2\Delta_{a^\dagger}$. Consequently, the optimal arm of $I'$ is $a^\dagger$ with margin $\Delta_{a^\dagger}$. Let $n_a$ denote the number of times that arm $a$ is pulled. We define the event
	\begin{align*}
		E=\{n_{a^\dagger}> T/2\}.
	\end{align*}
	Then, we have
	\begin{align*}
		\E_p\big[\@{Regret}_T\big]\geq \frac{T\Delta_{a^\dagger}}{2}\cdot p(E)
		,\quad
		\E_{p'}\big[\@{Regret}_T\big]\geq \frac{T\Delta_{a^\dagger}}{2}\cdot p'(E^\complement).
	\end{align*}
	Hence,
	\begin{align*}
		2CT^{1-\beta}
		\geq&\E_p\big[\@{Regret}_T\big]+\E_{p'}\big[\@{Regret}_T\big]\\
		\geq&\frac{T\Delta_{a^\dagger}}{2} \big(p(E)+p'(E^\complement)\big)\\
		=&\frac{T\Delta_{a^\dagger}}{2} \Big(1-\big(p'(E)-p(E)\big)\Big)\\
		\geq&\frac{T\Delta_{a^\dagger}}{2} \Big(1-\@{TV}\big(p,p'\big)\Big)\\
		\geq&\frac{T\Delta_{a^\dagger}}{2} \Big(1-\sqrt{1-\exp\big(-\@{KL}(p,p')\big)}\Big)\\
		\geq&\frac{T\Delta_{a^\dagger}}{2} \exp\left(-\frac{1}{2}\cdot\@{KL}(p,p')\right).
	\end{align*}
	By \Cref{lem:kl-decompose}, we have
	\begin{align*}
		\@{KL}(p,p')
		=&\sum_{a=1}^A\E_{p}[\-n_a]\cdot\@{KL}\big(\Pr(r\given a),\Pr'(r\given a)\big)\\
		=&\E_{p}[\-n_{a^\dagger}]\cdot\@{KL}\big(\Pr(r\given a^\dagger),\Pr'(r\given a^\dagger)\big)\\
		\leq&\E_{p}[\-n_{a^\dagger}]\cdot\Delta^2_{a^\dagger} \cdot 2/\zeta
	\end{align*}
	where  the last inequality is by \Cref{lem:kl-bern}. Putting the above two inequality together, we arrive at
	\begin{align*}
		\E_{p}[\-n_{a^\dagger}]\geq\frac{\zeta}{\Delta^2_{a^\dagger}}\log\left(\frac{\Delta_{a^\dagger}}{4 C T^{-\beta}}\right).
	\end{align*}
	This establishes a query lower bound for arm $a^\dagger$. Consequently, we have
	\begin{align*}
		\E[\@{Regret}_T]
		\geq 
		\sum_{a\neq a^\star} \E_{p}[\-n_{a}]\cdot\Delta_a
		\geq\sum_{a\neq a^\star}\frac{\zeta}{\Delta_{a}}\log\left(\frac{\Delta_{a}}{4 C T^{-\beta}}\right),
	\end{align*}
	and similarly,
	\begin{align*}
		\E[\@{Queries}_T]
		\geq 
		\sum_{a\neq a^\star} \E_{p}[\-n_{a}]
		\geq\sum_{a\neq a^\star}\frac{\zeta}{\Delta_{a}^2}\log\left(\frac{\Delta_{a}}{4 C T^{-\beta}}\right).
	\end{align*}
\end{proof}

Now we can proceed with the proof of \Cref{thm:lower-bound-stronger}.

\begin{proof}[Proof of \Cref{thm:lower-bound-stronger}]
	We provide a reduction from the multi-armed bandits with active queries to the contextual dueling bandits. Our desired lower bound for the contextual dueling bandit setting thus follows from the above lower bound for Multi-Armed Bandits (MABs). Let \Alg{} denote any algorithm for contextual dueling bandits. 
	
	\textbf{Reduction.}
	Since we focus on the multi-armed bandit where no context is involved, we just ignore the notation of context everywhere for brevity. We will start from an MAB instance, and then simulate a binary feedback and feed it to a dueling bandit algorithm \Alg{} which is used to solve the original MAB instance. Particularly, consider the MAB instance with A-many actions each with an expected reward denoted as $\bar r_a$. 

    At the beginning of iteration $t$ in the MAB instance, the learner calls the dueling algorithm \Alg{} to generate two actions $a_t$ and $b_t$. The learner plays $a_t$ at iteration $t$ to receive a reward $y_{a_t}$; the learner then moves to iteration $t+1$ to play $b_t$, and receives reward $y_{b_t}$. At the end of iteration $t+1$, the learner simulates a binary feedback by setting $o = 1$ if $y_{a_t} > y_{b_t}$; $o=-1$ if $y_{a_t} < y_{b_t}$; $o$ being $1$ or $-1$ uniform randomly if $y_{a_t}= y_{b_t}$. Then, the learner sends $(a_t, b_t, o)$ to the dueling algorithm \Alg{} to query for two actions which will be played at iterations $t+2$ and $t+3$, respectively.

	From the dueling algorithm \Alg{}'s perspective, given two actions $a$ and $b$, we can verify that the probability of seeing label 1 is $(\-r_a-\-r_b+1)/2$. So we can just specify the link function to be $\phi(d)=(d+1)/2$. As we verified earlier, the corresponding $\Phi$ is strongly convex~(\Cref{ex:sq-loss}). Moreover, since $f^\star(a,b) = \bar r_{a} - \bar r_{b}$, if we define the gap of the MAB instance as $\-\Delta\coloneqq\min_{a\neq a^\star}(\-r_{a^\star}-\-r_a)$ where $a^\star\coloneqq\argmax_i\-r_i$, then we have $\-\Delta=\Delta$ in this reduction where $\Delta$ is the definition of the gap in the dueling setting. We further note that the regret of the MAB instance is 
	$$
	\sum_{t=1}^T (\-r_{a^\star} -\-r_{a_t})+\sum_{t=1}^T (\-r_{a^\star} -\-r_{b_t}),
	$$
	which, by our definition of $f^\star$, is equivalent to the preference-based regret that occurred to the dueling algorithm \Alg{}. The number of queries is clearly equivalent as well. Thus, the regret and the query complexity of the dueling algorithm \Alg{} can be directly translated to the regret and the query complexity of the MAB instance. 
	
	Now, we are ready to prove the two claims in our statement.
	
	\textbf{Proof of the first claim.}
	We refer the reader to \citet[Theorem 15.2]{lattimore2020bandit} for a proof of the minimax regret lower bound of $\Omega(\sqrt{AT})$ for the MAB. Through the reduction outlined above, that lower bound naturally extends to the dueling bandits setting, yielding $\@{Regret}^{\@{CB}}_T\geq\Omega(\sqrt{AT})$ (otherwise, via the above reduction, we would have achieved an approach that breaks the lower bound of MAB).

	\textbf{Proof of the second claim.}
We choose an arbitrary MAB for which $\zeta=\min_a\min\{\-r_a,1-\-r_a\}>0.2$ and the gaps of all arms are equal to $\Delta$. Invoking~\Cref{lem:mab-tradeoff}, we have
	\begin{align*}
		\E\big[\@{Regret}_T\big]&\geq 
		\frac{0.2(A-1)}{\Delta}\log\left(\frac{\Delta}{4CT^{-\beta}}\right)\geq \Omega\left(\frac{A}{\Delta} \right),\\
		\E\big[\@{Queries}_T\big]&\geq 
		\frac{0.2(A-1)}{\Delta^2}\log\left(\frac{\Delta}{4CT^{-\beta}}\right)\geq\Omega\left(\frac{A}{\Delta^2}\right).
	\end{align*}
	We further choose $\Delta=40CT^{-\beta}$ and $C=\sqrt{A}$, leading to
	\begin{align*}
		\E\big[\@{Regret}_T\big]
		&\geq 
		\frac{0.2(A-1)}{40\sqrt{A}}\cdot T^{\beta}
		=
		\Omega\left(\sqrt{A}\cdot T^\beta\right)
            ,\\
		\E\big[\@{Queries}_T\big]
		&\geq 
		\frac{0.2(A-1)}{1600A}\cdot T^{2\beta}
		=\Omega\left(T^{2\beta}\right).
	\end{align*}
	
	Via the reduction we have shown above, these lower bounds naturally extend to the contextual dueling bandit setting, thereby completing the proof.
\end{proof}

\subsubsection{Alternative Lower Bounds Conditioning on the Limit of Regret}\label{sec:lb-2}

In this section, we establish an analogue of \Cref{thm:lower-bound-stronger} but under a different condition. We first introduce the concept of \textit{diminishing regret}.

\begin{definition}
	We say that an algorithm guarantees a diminishing regret if for all contextual dueling bandit instances and $p>0$, it holds that $$\lim_{T\rightarrow\infty}\frac{\E[\@{Regret}^{\@{CB}}_T]}{T^p}=0.$$
\end{definition}

The lower bounds under the assumption of diminishing regret guarantees are stated as follows.

\begin{theorem}[Lower bounds]\label{thm:lower-bound2}
The following two claims hold:
\begin{enumerate}
	\item[(1)] for any algorithm, there exists an instance that leads to $\@{Regret}^{\@{CB}}_T\geq\Omega(\sqrt{AT})$;
	\item[(2)] for any gap $\Delta$ and any algorithm achieving diminishing regret, there exists an instance with gap $\Delta$ that results in $\E[\@{Regret}^{\@{CB}}_T]\geq\Omega(A/\Delta)$ and $\E[\@{Queries}^{\@{CB}}_T]\geq\Omega(A/\Delta^2)$ for sufficiently large $T$.
\end{enumerate}
\end{theorem}

We should highlight that the condition of diminishing regret (\Cref{thm:lower-bound2}) and the worst-case regret upper bounds (\Cref{thm:lower-bound,thm:lower-bound-stronger}) are not comparable in general. However, \Cref{thm:lower-bound2} is also applicable to our algorithm (\Cref{alg:cb}) since our algorithm possesses an instance-dependent regret upper bound that is clearly diminishing.

To prove \Cref{thm:lower-bound2}, we first show the following lemma, which is a variant of \Cref{lem:mab-tradeoff}.
\begin{lemma}\label{lem:mab-tradeoff2}
	Let $\+I$ denote the set of all MAB instances. Assume \Alg{} is an algorithm that achieves diminishing regret for all MAB instances in $\+I$, i.e., for any $I\in\+I$ and $p>0$, it holds that
	\begin{equation*}
		\lim_{T\rightarrow\infty}\frac{\E[\@{Regret}_T]}{T^p}=0.
	\end{equation*}
	 Then, for any MAB instance $I\in\+I$, the regret and the number of queries made by algorithm \Alg{} are lower bounded in the following manner:
	\begin{equation*}
		\mathop{\lim\inf}_{T\rightarrow\infty}\frac{\E\big[\@{Regret}_T\big]}{\log T}\geq\sum_{a\neq a^\star}\frac{\zeta}{\Delta_{a}}
		,\quad
		\mathop{\lim\inf}_{T\rightarrow\infty}\frac{\E\big[\@{Queries}_T\big]}{\log T}\geq 
		\sum_{a\neq a^\star}\frac{\zeta}{\Delta^2_{a}}
	\end{equation*}
	where the coefficient $\zeta\coloneqq\min_a\min\{\-r_a,1-\-r_a\}$ depends on the instance $I$. Recall that $\@{Regret}_T$ and $\@{Queries}_T$ are defined in \eqref{eq:mab-regret-query}.
\end{lemma}

\begin{proof}[Proof of \Cref{lem:mab-tradeoff2}]
The proof is similar to \Cref{lem:mab-tradeoff}. 
For any MAB instance $I\in\+I$ and any arm $a^\dagger$, we define a corresponding MAB instance $I'$ as follows. Denote $\-r$ and $\-r'$ as the mean reward of $I$ and $I'$, respectively. For $I'$, we set the mean reward $\-r'_a=\-r_a$ for any $a\neq a^\dagger$ and $\-r'_{a^\dagger}=\-r_{a^\dagger}+2\Delta_{a^\dagger}$. Consequently, the optimal arm of $I'$ is $a^\dagger$ with margin $\Delta_{a^\dagger}$. Let $n_a$ denote the number of times that arm $a$ is pulled. We define the event
	\begin{align*}
		E=\{n_{a^\dagger}> T/2\}.
	\end{align*}
	Let $p$ and $p'$ denote the probability of $I$ and $I'$, respectively. Then, we have
	\begin{align*}
		\E_p\big[\@{Regret}_T\big]\geq \frac{T\Delta_{a^\dagger}}{2}\cdot p(E)
		,\quad
		\E_{p'}\big[\@{Regret}_T\big]\geq \frac{T\Delta_{a^\dagger}}{2}\cdot p'(E^\complement)
	\end{align*}
	where $E^\complement$ means the complement of event $E$. Hence,
	\begin{align*}
		\E_p\big[\@{Regret}_T\big]+\E_{p'}\big[\@{Regret}_T\big]
		\geq&\frac{T\Delta_{a^\dagger}}{2} \big(p(E)+p'(E^\complement)\big)\\
		=&\frac{T\Delta_{a^\dagger}}{2} \Big(1-\big(p'(E)-p(E)\big)\Big)\\
		\geq&\frac{T\Delta_{a^\dagger}}{2} \Big(1-\@{TV}\big(p,p'\big)\Big)\\
		\geq&\frac{T\Delta_{a^\dagger}}{2} \Big(1-\sqrt{1-\exp\big(-\@{KL}(p,p')\big)}\Big)\\
		\geq&\frac{T\Delta_{a^\dagger}}{2} \exp\left(-\frac{1}{2}\cdot\@{KL}(p,p')\right).
	\end{align*}
	Here $\@{TV}$ denotes the total variation distance. By \Cref{lem:kl-decompose}, we have
	\begin{align*}
		\@{KL}(p,p')
		=&\sum_{a=1}^A\E_{p}[\-n_a]\cdot\@{KL}\big(\Pr(r\given a),\Pr'(r\given a)\big)\\
		=&\E_{p}[\-n_{a^\dagger}]\cdot\@{KL}\big(\Pr(r\given a^\dagger),\Pr'(r\given a^\dagger)\big)\\
		\leq&\E_{p}[\-n_{a^\dagger}]\cdot\Delta^2_{a^\dagger} \cdot 2/\zeta
	\end{align*}
	where  the last inequality is by \Cref{lem:kl-bern}. Putting it all together, we arrive at
	\begin{align*}
		\E_{p}[\-n_{a^\dagger}]\geq\frac{\zeta}{\Delta^2_{a^\dagger}}\log\left(\frac{T\Delta_{a^\dagger}}{2\Big(\E_p\big[\@{Regret}_T\big]+\E_{p'}\big[\@{Regret}_T\big]\Big)}\right).
	\end{align*}
	Taking the limit on both sides yields
	\begin{align*}
		\mathop{\lim\inf}_{T\rightarrow\infty}
		\frac{\E_{p}[\-n_{a^\dagger}]}{\log T}
		\geq&\mathop{\lim\inf}_{T\rightarrow\infty}\frac{\zeta}{\Delta^2_{a^\dagger}}\cdot\frac{\log\left(\frac{T\Delta_{a^\dagger}}{2\Big(\E_p\big[\@{Regret}_T\big]+\E_{p'}\big[\@{Regret}_T\big]\Big)}\right)}{\log T}\\
		=&\mathop{\lim\inf}_{T\rightarrow\infty}\frac{\zeta}{\Delta^2_{a^\dagger}}\cdot\left(1+\underbrace{\frac{\log(\Delta_{a^\dagger}/2)}{\log T}}_{\rm(i)}-\underbrace{\frac{\log\Big(\E_p\big[\@{Regret}_T\big]+\E_{p'}\big[\@{Regret}_T\big]\Big)}{\log T}}_{\rm(ii)}\right).
	\end{align*}
	Here the limit of $\rm(i)$ is clearly 0. For the limit of $\rm(ii)$, we note that by the definition of diminishing regret, for any $C>0$, there exists a $T'$ such that $\E[\@{Regret}_{T}]/T^p\leq C$ for any $T>T'$. This implies
	\begin{align*}
		\frac{\log\Big(\E_p\big[\@{Regret}_T\big]+\E_{p'}\big[\@{Regret}_T\big]\Big)}{\log T}
		\leq
		\frac{\log\Big( 2CT^p\Big)}{\log T}=\frac{\log(2C)}{\log T}+p
	\end{align*}
	for any $p>0$. Therefore, the limit of $\rm(ii)$ is also 0. Plugging these back, we obtain
	\begin{align*}
		\mathop{\lim\inf}_{T\rightarrow\infty}
		\frac{\E_{p}[\-n_{a^\dagger}]}{\log T}\geq\frac{\zeta}{\Delta^2_{a^\dagger}}.
	\end{align*}
	This establishes a query lower bound for arm $a^\dagger$. Consequently, we have
	\begin{align*}
		\mathop{\lim\inf}_{T\rightarrow\infty}
		\frac{\E[\@{Regret}_T]}{\log T}
		\geq 
		\mathop{\lim\inf}_{T\rightarrow\infty}
		\sum_{a\neq a^\star} \frac{\E_{p}[\-n_{a}]\cdot\Delta_a}{\log T}
		\geq\sum_{a\neq a^\star}\frac{\zeta}{\Delta_{a}},
	\end{align*}
	and similarly,
	\begin{align*}
		\mathop{\lim\inf}_{T\rightarrow\infty}
		\frac{\E[\@{Queries}_T]}{\log T}
		\geq 
		\mathop{\lim\inf}_{T\rightarrow\infty}
		\sum_{a\neq a^\star} \frac{\E_{p}[\-n_{a}]}{\log T}
		\geq\sum_{a\neq a^\star}\frac{\zeta}{\Delta_{a}^2}.
	\end{align*}
\end{proof}

Now, we proceed with the proof of \Cref{thm:lower-bound2}.

\begin{proof}[Proof of \Cref{thm:lower-bound2}]
The proof of the first claim is the same as \Cref{thm:lower-bound-stronger}, so we will omit it here. Let us now focus on the proof of the second claim. By \Cref{lem:mab-tradeoff2}, for any algorithm achieving diminishing regret, the following is true for any MAB instance:
	\begin{equation*}
		\mathop{\lim\inf}_{T\rightarrow\infty}\frac{\E\big[\@{Regret}_T\big]}{\log T}\geq 
		\sum_{a\neq a^\star}\frac{\zeta}{\Delta_{a}},\quad
		\mathop{\lim\inf}_{T\rightarrow\infty}\frac{\E\big[\@{Queries}_T\big]}{\log T}\geq 
		\sum_{a\neq a^\star}\frac{\zeta}{\Delta^2_{a}}.
	\end{equation*}
	We choose an arbitrary MAB for which $\zeta\geq 0.2$ and the gaps of all suboptimal arms are equal to $\Delta$. Then, for this instance, we have
	\begin{equation*}
		\mathop{\lim\inf}_{T\rightarrow\infty}\frac{\E\big[\@{Regret}_T\big]}{\log T}\geq 
		\frac{0.2(A-1)}{\Delta},\quad
		\mathop{\lim\inf}_{T\rightarrow\infty}\frac{\E\big[\@{Queries}_T\big]}{\log T}\geq 
		\frac{0.2(A-1)}{\Delta^2}.
	\end{equation*}
	By the definition of limit, when $T$ is large enough (exceeding a certain threshold), we have
	\begin{equation*}
		\frac{\E\big[\@{Regret}_T\big]}{\log T}\geq 
		\frac{0.1(A-1)}{\Delta},\quad
		\frac{\E\big[\@{Queries}_T\big]}{\log T}\geq 
		\frac{0.1(A-1)}{\Delta^2}.
	\end{equation*}
	Via the reduction we have shown in the proof of \Cref{thm:lower-bound-stronger}, these lower bounds naturally extend to the contextual dueling bandit setting, thereby completing the proof.
\end{proof}

\subsection{Proof of Theorem~\ref{thm:cb-general-regret}}\label{sec:pf-thm-cb-general-regret}
\begin{proof}[Proof of \Cref{thm:cb-general-regret}]

We establish the bounds for regret and the number of queries, consecutively. First, we set an arbitrary gap threshold $\epsilon > 0$. Since our algorithm is independent of $\epsilon$, we can later choose any $\epsilon$ that minimizes the upper bounds.

\textbf{Proof of regret.} 
We start with the regret upper bound. By definition, we have
\begin{align*}
	\@{Regret}^{\@{CB}}_T=\sum_{t=1}^{T}\big(
	f^\star(x_t,\pi_{f^\star}(x_t),a_t)+
				f^\star(x_t,\pi_{f^\star}(x_t),b_t)\big)	.
\end{align*}
Since $a_t$ and $b_t$ are always drawn independently from the same distribution in \Cref{alg:cb}, we only need to consider the regret of the $a_t$ part in the following proof for brevity --- multiplying the result by two would yield the overall regret.

The worst-case regret upper bound presented in  \Cref{lem:worst-case-regret-ub} doesn't reply on the gap assumption and thus remains applicable in this setting. Hence, we only need to prove the instance-dependent regret upper bound. To that end, we first need an analogue of \Cref{lem:lambda-all-0}.

\begin{lemma}\label{lem:lambda-all-0-general}
	Fix any $\epsilon > 0$. Whenever 
	$$
	2T_\epsilon + 56A^2\beta\cdot\frac{\@{dim}_E\left(\+F,\epsilon\right)}{\epsilon}\cdot\log(2/(\delta\epsilon))
	<
	\sqrt{AT/\beta},
	$$
	we have $\lambda_1=\lambda_2=\dots=\lambda_T=0$ with probability at least $1-\delta$.
\end{lemma}
\begin{proof}[Proof of \Cref{lem:lambda-all-0-general}]
	The proof is similar to \Cref{lem:lambda-all-0} and is via contradiction. Assume the inequality holds but there exists $t'$ for which $\lambda_{t'}=1$. Without loss of generality, we assume that $\lambda_t=0$ for all $t<t'$, namely that $t'$ is the first time that $\lambda_t$ is 1. Then by definition of $\lambda_{t'}$, we have
	\begin{align*}
		\sum_{s=1}^{t'-1} Z_sw_s\geq\sqrt{AT/\beta}.
	\end{align*}
	On the other hand, we have
	\begin{align*}
	\sum_{s=1}^{t'-1} Z_s w_s
	= &
	\sum_{s=1}^{t'-1} \indic\{\@{Gap}(x_t)\leq\epsilon\} Z_s w_s + 	\sum_{s=1}^{t'-1} \indic\{\@{Gap}(x_t)>\epsilon\} Z_s w_s\\
	\leq &
	2T_\epsilon + 56A^2\beta\cdot\frac{\@{dim}_E\left(\+F,\epsilon\right)}{\epsilon}\cdot\log(2/(\delta\epsilon))
	\end{align*}
	where the inequality is by \Cref{lem:sum-zw}. The above two inequalities contradicts with the conditions.
\end{proof}
Towards an instance-dependent regret upper bound, we adapt the proof of \Cref{lem:ins-depend-regret-ub} to this setting. We consider two cases. First, when 
	\begin{equation}\label{eq:regret-cases1-general}
	2T_\epsilon + 56A^2\beta\cdot\frac{\@{dim}_E\left(\+F,\epsilon\right)}{\epsilon}\cdot\log(2/(\delta\epsilon))
	<
	\sqrt{AT/\beta},
	\end{equation}
	we invoke \Cref{lem:lambda-all-0-general} and get that $\lambda_t=0$ for all $t\in[T]$. Hence, we have
	\begin{align*}
		\@{Regret}^{\@{CB}}_T
		=&\sum_{t=1}^{T}\big(
		f^\star(x_t,\pi_{f^\star}(x_t),a_t)+
				f^\star(x_t,\pi_{f^\star}(x_t),b_t)\big)\\
		\leq&2\sum_{t=1}^T  \indic\{\@{Gap}(x_t)\leq\epsilon\} Z_t w_t + 2 \sum_{t=1}^T  \indic\{\@{Gap}(x_t)>\epsilon\} Z_t w_t\\
		\leq& 4 T_\epsilon + 112A^2\beta\cdot\frac{\@{dim}_E\left(\+F,\epsilon\right)}{\epsilon}\cdot\log(2/(\delta\epsilon))\\
		\leq& 136\beta\cdot\log(4\delta^{-1})\cdot T_\epsilon + 3808 A^2\beta^2\cdot\frac{\@{dim}_E\left(\+F,\epsilon\right)}{\epsilon}\cdot\log^2(4/(\delta\epsilon))
	\end{align*}
	where the first inequality is by \Cref{lem:regret-bounded-by-w} and the fact that we incur no regret when $Z_t=0$ since $f^\star\in\+F_t$. The second inequality is by \Cref{lem:sum-zw}.
	
	On the other hand, when the contrary of \eqref{eq:regret-cases1-general} holds, i.e., 
	\begin{equation}\label{eq:regret-cases2-general}
	2T_\epsilon + 56A^2\beta\cdot\frac{\@{dim}_E\left(\+F,\epsilon\right)}{\epsilon}\cdot\log(2/(\delta\epsilon))
	\geq
	\sqrt{AT/\beta},
	\end{equation}
	applying \Cref{lem:worst-case-regret-ub}, we have
	\begin{align*}
		\@{Regret}^{\@{CB}}_T
		\leq&68\sqrt{AT\beta}\cdot\log(4\delta^{-1})\\
		=&68\beta\cdot\log(4\delta^{-1})\cdot\sqrt{AT/\beta}\\
		\leq&68\beta\cdot\log(4\delta^{-1})\cdot
	\left(2 T_\epsilon + 56A^2\beta\cdot\frac{\@{dim}_E\left(\+F,\epsilon\right)}{\epsilon}\cdot\log(2/(\delta\epsilon))\right)\\
		\leq& 136\beta\cdot\log(4\delta^{-1})\cdot T_\epsilon + 3808 A^2\beta^2\cdot\frac{\@{dim}_E\left(\+F,\epsilon\right)}{\epsilon}\cdot\log^2(4/(\delta\epsilon))
	\end{align*}
	where we apply the condition \eqref{eq:regret-cases2-general} in the second inequality.

\textbf{Proof of the number of queries.}
To show an upper bound for the number of queries, we also consider two cases. First, when
\begin{align}\label{eq:query-case-general}
	2T_\epsilon + 56A^2\beta\cdot\frac{\@{dim}_E\left(\+F,\epsilon\right)}{\epsilon}\cdot\log(2/(\delta\epsilon))
	<
	\sqrt{AT/\beta},
\end{align}
we can invoke \Cref{lem:lambda-all-0-general} and get that $\lambda_t=0$ for all $t\in[T]$. Hence, similar to the proof of \Cref{lem:query-bound}, we have
\begin{align*}
	\@{Queries}^{\@{CB}}_T
	=&\sum_{t=1}^T Z_t\\
	=&\sum_{t=1}^T Z_t\indic\{\@{Gap}(x_t) < \epsilon\}+\sum_{t=1}^T Z_t\indic\{\@{Gap}(x_t) \geq \epsilon\}\\
	=&T_\epsilon + \sum_{t=1}^TZ_t\sup_{a,b\in\+A_t}\indic\left\{\sup_{f,f'\in\+F_t}f(x_t,a,b)-f'(x_t,a,b)\geq\epsilon\right\}\\
	\leq& T_\epsilon +\sum_{t=1}^TZ_t\sum_{a,b}\indic\left\{\sup_{f,f'\in\+F_t}f(x_t,a,b)-f'(x_t,a,b)\geq\epsilon\right\}\\
	\leq& T_\epsilon +A^2\underbrace{\sum_{t=1}^TZ_t \E_{a,b\sim p_t}\indic\left\{\sup_{f,f'\in\+F_t}f(x_t,a,b)-f'(x_t,a,b)\geq\epsilon\right\}}_{(*)}
\end{align*}
where the second inequality holds as $p_t(a)$ is uniform for any $a,b$ when $\lambda_t=0$. We apply \Cref{lem:freedman} and \Cref{lem:w-eluder} to $(*)$ and obtain
\begin{align*}
	(*)\leq& 2\sum_{t=1}^TZ_t \indic\left\{\sup_{f,f'\in\+F_t}f(x_t,a_t,b_t)-f'(x_t,a_t,b_t)\geq\epsilon\right\}+8\log(\delta^{-1})\\
	\leq& 2\left(\frac{4\beta}{\epsilon^2}+1\right)\@{dim}_E(\+F;\epsilon)+8\log(\delta^{-1})\\
	\leq& \frac{10\beta}{\epsilon^2}\cdot\@{dim}_E(\+F;\epsilon)+8\log(\delta^{-1}).
\end{align*}
Plugging this back, we obtain
\begin{align*}
	\@{Queries}^{\@{CB}}_T
	\leq& T_\epsilon + \frac{10A^2\beta}{\epsilon^2}\cdot\@{dim}_E(\+F;\epsilon)+8A^2\log(\delta^{-1})\\
	\leq& 8 T^2_\epsilon\beta/A + 6272 A^3 \beta^3 \frac{\@{dim}^2_E\left(\+F,\epsilon\right)}{\epsilon^2}\cdot\log^2(2/(\delta\epsilon))
\end{align*}
where the second line corresponds to the upper bound derived from the alternative case, which is shown below.

When the contrary of \eqref{eq:query-case-general} holds, i.e., 
\begin{align*}
	2T_\epsilon + 56 A^2\beta\cdot\frac{\@{dim}_E\left(\+F,\epsilon\right)}{\epsilon}\cdot\log(2/(\delta\epsilon))
	\geq
	\sqrt{AT/\beta}.
\end{align*}
Squaring both sides and leveraging the inequality $(a+b)^2\leq 2a^2+2b^2$, we obtain
\begin{align*}
8T^2_\epsilon + 6272 A^4 \beta^2 \frac{\@{dim}^2_E\left(\+F,\epsilon\right)}{\epsilon^2}\cdot\log^2(2/(\delta\epsilon))
\geq AT/\beta
\end{align*}
which leads to
\begin{equation*}
T
\leq 8 T^2_\epsilon\beta/A + 6272 A^3 \beta^3 \frac{\@{dim}^2_E\left(\+F,\epsilon\right)}{\epsilon^2}\cdot\log^2(2/(\delta\epsilon)).
\end{equation*}
We note that we always have $\@{Queries}^{\@{CB}}_T\leq T$ and thus 
\begin{equation*}
\@{Queries}^{\@{CB}}_T
\leq
T
\leq 8 T^2_\epsilon\beta/A + 6272 A^3 \beta^3 \frac{\@{dim}^2_E\left(\+F,\epsilon\right)}{\epsilon^2}\cdot\log^2(2/(\delta\epsilon)).
\end{equation*}

\textbf{Minimizing on $\epsilon$.}
Given that the aforementioned proofs hold for any threshold $\epsilon$, we can select the specific value of $\epsilon$ that minimizes the upper bounds. Hence, we deduce the desired result.
\end{proof}

\subsection{Proof of Theorem~\ref{thm:il-regret}}\label{sec:pf-thm-il-regret}
\begin{proof}[Proof of Theorem~\ref{thm:il-regret}]
The upper bound of the number of queries is straightforward: \Cref{alg:il} is simply running $H$ instances of \Cref{alg:cb}, so the total number of queries is simply the sum of these $H$ instances. For bounding the regret, we have
	\begin{align*}
		\@{Regret}^{\@{IL}}_T
		=&\sum_{t=1}^T V^{\pi_e}_0(x_{t,0})-V^{\pi_t}_0(x_{t,0})\\
		\leq&\sum_{h=0}^{H-1} \sum_{t=1}^T \E_{x_{t,h},a_{t,h}\sim d^{\pi_t}_{x_{t,0},h}}\Big[Q^{\pi_e}_h(x_{t,h},\pi^{\pi_e}_h(x_{t,h}))-Q^{\pi_e}_h(x_{t,h},a_{t,h})\Big]\\
		\leq&\sum_{h=0}^{H-1} \sum_{t=1}^T \E_{x_{t,h},a_{t,h}\sim d^{\pi_t}_{x_{t,0},h}}\Big[Q^{\pi_e}_h(x_{t,h},\pi_h^+(x_{t,h}))-Q^{\pi_e}_h(x_{t,h},a_{t,h})\Big]\\
		&\quad-\sum_{h=0}^{H-1} \sum_{t=1}^T \E_{x_{t,h}\sim d^{\pi_t}_{x_{t,0},h}}\Big[A^{\pi_e}_h(x_{t,h},\pi^+_h(x_{t,h}))\Big]\\
		\leq&H\cdot\E\left[\@{Regret}^{\@{CB}}_T\right]-\@{Adv}_T.
	\end{align*}
	where the first inequality holds by \Cref{lem:pdl}, and we denote $\pi^+_h(x_{t,h})=\argmax_a Q^{\pi_e}_h(x_{t,h},a)$ in the second inequality. Then, we can plug the upper bound of $\@{Regret}^{\@{CB}}_T$ (\Cref{thm:cb-regret}). Moreover, we need to take a union bound over all $h \in [H]$.
\end{proof}

\end{document}